\documentclass[10pt]{article}

\usepackage[utf8]{inputenc} 
\usepackage[T1]{fontenc}

\usepackage{url}          
\usepackage{booktabs}      
\usepackage{amsfonts}       
\usepackage{nicefrac}       
\usepackage{microtype}      
\usepackage{fullpage}
\usepackage{mathtools}
\usepackage{xcolor}
\newcommand*{\colorboxed}{}
\def\colorboxed#1#{
	\colorboxedAux{#1}
}
\newcommand*{\colorboxedAux}[3]{
	\begingroup
	\colorlet{cb@saved}{.}
	\color#1{#2}
	\boxed{
		\color{cb@saved}
		#3
	}
	\endgroup
}

\usepackage{amssymb}
\usepackage{amsmath,amsthm}

\usepackage{enumitem}

\usepackage{algorithm}
\usepackage{algorithmic}

\usepackage{xcolor,colortbl}
\definecolor{DarkBlue}{RGB}{22,54,93}

\usepackage[colorlinks = true]{hyperref}        
\usepackage{cleveref}
\hypersetup{linkcolor =red, citecolor = blue}

\usepackage{natbib}
\usepackage{graphicx}
\usepackage{subfigure}
\usepackage{adjustbox}
\usepackage{amsmath}
\usepackage{amssymb}
\usepackage{mathtools}
\usepackage{amsthm}
\usepackage{threeparttable}

\theoremstyle{plain}
\newtheorem{theorem}{Theorem}[section]
\newtheorem{proposition}[theorem]{Proposition}
\newtheorem{lemma}[theorem]{Lemma}
\newtheorem{corollary}[theorem]{Corollary}
\theoremstyle{definition}
\newtheorem{definition}[theorem]{Definition}
\newtheorem{assumption}[theorem]{Assumption}
\theoremstyle{remark}

\usepackage[textsize=tiny]{todonotes}

\allowdisplaybreaks[4]

\date{}
\author
{
        Songtao Feng\thanks{\small Department of Electrical and Computer Engineering, The Ohio State University, OH 43210, USA; e-mail: {\tt   feng.1359@osu.edu}}
        ,~~~Ming Yin\thanks{\small Department of Computer Science, UC Santa Barbara, CA 93106, USA; e-mail: {\tt  ming\_yin@ucsb.edu}}
        ,~~~Ruiquan Huang\thanks{\small School of Electrical Engineering and Computer Science, The Pennsylvania State University, University Park, PA 16802, USA; e-mail: {\tt  rzh5514@psu.edu}}
        ,
        Yu-Xiang Wang\thanks{\small Department of Computer Science, UC Santa Barbara, CA 93106, USA; e-mail: {\tt   yuxiangw@cs.ucsb.edu}}
	,\\~~~Jing Yang\thanks{\small School of Electrical Engineering and Computer Science, The Pennsylvania State University, University Park, PA 16802, USA; e-mail: {\tt  yangjing@psu.edu }}
	,~~~Yingbin Liang\thanks{\small Department of Electrical and Computer Engineering, The Ohio State University, OH 43210, USA; e-mail: {\tt   liang.889@osu.edu}}
}

 \usepackage[utf8]{inputenc} 
 \usepackage[T1]{fontenc}    
 \usepackage{hyperref}       
 \usepackage{url}            
 \usepackage{booktabs}       
 \usepackage{amsfonts}       
 \usepackage{nicefrac}      
 \usepackage{microtype}      
 \usepackage{xcolor}        
 
 \usepackage[utf8]{inputenc} 
 \usepackage[T1]{fontenc}

 \usepackage{url}           
 \usepackage{booktabs}       
 \usepackage{amsfonts}     
 \usepackage{nicefrac}      
 \usepackage{microtype}      
 \usepackage{mathtools}
 \usepackage{xcolor}

 \usepackage{amssymb}
 \usepackage{amsmath,amsthm}

 \usepackage{enumitem}
 
 \usepackage{algorithm}
 \usepackage{algorithmic}
 
 \usepackage{xcolor,colortbl}
 \definecolor{DarkBlue}{RGB}{22,54,93}

 \usepackage{cleveref}
 
 \usepackage{natbib}
 \usepackage{graphicx}
 \usepackage{subfigure}
 \usepackage{adjustbox}
 \usepackage{mathrsfs}

 \usepackage{amsmath}
 \usepackage{amssymb}
 \usepackage{mathtools}
 \usepackage{amsthm}
 \usepackage{threeparttable}

 \theoremstyle{plain}
 
\crefname{assumption}{assumption}{assumptions}

 \usepackage[textsize=tiny]{todonotes}
 \newcommand{\Pb}{\mathbb{P}}
 \newcommand{\Rb}{\mathbb{R}}

 \newcommand{\Eb}{\mathbb{E}}

 \newcommand{\Sc}{\mathcal{S}}
 \newcommand{\Ac}{\mathcal{A}}
 \newcommand{\Dc}{\mathcal{D}}
 \newcommand{\Nc}{\mathcal{N}}

 \newcommand{\URM}[1]{\left(\mathrm{\uppercase\expandafter{\romannumeral#1}}\right)}

 \newcommand{\norm}[1]{\left\lVert#1\right\rVert}
 \newcommand{\dif}{\mathop{}\!\mathrm{d}}

 \newcommand{\TV}{\mathrm{TV}}
\newcommand{\Tc}{\mathcal{T}}
\newcommand{\Oc}{\mathcal{O}}
\newcommand{\Var}{\mathrm{Var}}

\newcommand{\DR}{{\mathrm{D-Regret}}}
\newcommand{\Jb}{\mathbb{J}}
\newcommand{\Fc}{\mathcal{F}}
\newcommand{\Gc}{\mathcal{G}}
\newcommand{\Bc}{\mathcal{B}}
\newcommand{\Lc}{\mathcal{L}}

\newcommand{\stark}{{(*,k)}}

\newcommand{\Int}[1]{  {\kern0pt#1}^{\mathrm{o}} }
\newcommand{\GDE}{\dim_{\mathrm{DE}}}
\newcommand{\GBE}{\dim_{\mathrm{DBE}}}
\newcommand{\BE}{\dim_{\mathrm{BE}}}

\newcommand{\trace}{\mathrm{tr}}

\allowdisplaybreaks[4]

\newcommand{\alg}{SW-OPEA }
\newcommand{\algg}{SW-OPEA}

\makeatletter
\newcommand{\uset}[3][0ex]{
  \mathrel{\mathop{#3}\limits_{
    \vbox to#1{\kern -5\ex@
    \hbox{$#2$}\vss}}}}
\makeatother

\allowdisplaybreaks[4]

\title{Non-stationary Reinforcement Learning under General Function Approximation\footnote{To appear at International Conference on Machine Learning (ICML), 2023.}
}
 
\begin{document}
\maketitle
\begin{abstract}
General function approximation is a powerful tool to handle large state and action spaces in a broad range of reinforcement learning (RL) scenarios. However, theoretical understanding of non-stationary MDPs with general function approximation is still limited. In this paper, we make the first such an attempt. We first propose a new complexity metric called dynamic Bellman Eluder (DBE) dimension for non-stationary MDPs, which subsumes majority of existing tractable RL problems in static MDPs as well as non-stationary MDPs. Based on the proposed complexity metric, we propose a novel confidence-set based model-free algorithm called SW-OPEA, which features a sliding window mechanism and a new confidence set design for non-stationary MDPs. We then establish an upper bound on the dynamic regret for the proposed algorithm, and show that SW-OPEA is provably efficient as long as the variation budget is not significantly large. We further demonstrate via examples of non-stationary linear and tabular MDPs that our algorithm performs better in small variation budget scenario than the existing UCB-type algorithms. To the best of our knowledge, this is the first dynamic regret analysis in non-stationary MDPs with general function approximation.
\end{abstract}

\section{Introduction}
Reinforcement learning (RL) commonly refers to the sequential decision making framework modeled by a Markov Decision Process (MDP), where agent aims to maximize its cumulative reward in an unknown environment~\citep{Sutton1998}. RL has achieved great success in a variety of practical applications, including games~\cite{silver:2016,silver:2017,silver:2018,Vinyals:game:2019}, robotics~\citep{Kober:robo:2013,gu:robo:2017}, and autonomous driving~\citep{Yurtsever:auto_drive:2019}. So far, most existing RL works have focused on a {\em static} MDP model, in which both the reward and the transition kernel are time-invariant. However, {\em non-stationarity}\footnote{We emphasize non-stationarity is different from time-inhomogeneity (\emph{e.g.} \cite{yin2021near}). The latter allows transition $P_t$ to be different for $t\in[H]$, but $P_t$'s are fixed across episodes.} naturally occurs in many sequential decision problems such as online advertisement auctions~\citep{Cai:adv:2017,Lu:adv:2019}, traffic management~\citep{Chen:traffic:2020}, health-care operations~\citep{Shortreed:med:2010}, and inventory control~\citep{Agrawal:inventory:2019}. 

Compared to static RL, non-stationary RL can be significantly challenging. Under the general non-stationary environment, designing algorithm that achieve sublinear regret might not be possible due to the worst scenario where rewards and transitions change drastically over time.
A line of extensive studies have focused on tabular non-stationary MDPs~\citep{Peter_auer:adversarial_MDP:NeurIPS:2008,Pratik_Gajane:adversarial_MDP_arXiv:2018,Even-Dar:MDP:2009,Yu_Jia:MDP:2009,Jia_Yuanyu:MDP:2009,Neu:MDP:MDP:2010,Neu:MDP:MDP:2012,Zimin:MDP:2013,Dekel:MDP:2013,Rosenberg:MDP:2019,Chi_Jin:MDP:2020,Cheung:adversarial_MDP:PMLR:2020,Yingjie_Fei:adversarial_MDP:NeurIPS:2020,Weichao_Mao:adversarial_MDP:2021}. 
However, the performance of these algorithms suffers from large number of states in non-stationary MDPs, which precludes its usage in exponentially large or even continuous state spaces. Therefore, function approximation has become a prominent tool and several works proposed algorithms for non-stationary MDPs with structural assumptions, such as state-action forming a metric space~\citep{Domingues:adversarial_MDP:2020}, linear MDPs~\cite{Huozhi_Zhou:adversarial_MDP:2022,Ahmed_Touati:adversarial_MDP:2020}, linear mixture MDPs~\cite{Han_Zhong:adversarial_MDP:2021}. However, the structural function approximation of (such as linear) MDPs typically restrict the designed algorithms to perform well only under limited classes of MDPs, and may not be applicable generally. This naturally leads to the following open question:

\begin{center}
\textit{Can we design an algorithm that achieves a desired regret performance \\for non-stationary MDPs under general function approximation?}
\end{center}

To this end, there are a few challenges. (a) We need to identify an appropriate complexity metric for non-stationary MDPs that covers many existing problems of interest; (b) We need to design an algorithm that can handle non-stationary without function structures on transition kernels and rewards to exploit; and (c) Establishing a dynamic regret bound that potentially improves those for non-stationary simpler  MDPs such as linear and tabular cases is non-trivial. In this paper, we give an affirmative answer to the above question by addressing the aforementioned challenges. 

We summarize our contributions as follows.

\noindent{\bf Complexity metric:} We propose a new complexity metric named the Dynamic Bellman Eluder (DBE) dimension for non-stationary MDPs, which generalizes the Bellman Eluder (BE) dimension designed for stationary MDPs~\cite{Chi_Jin:bellman_eluder:2020}, and subsumes a broad class of RL problems including low BE dimension problems in stationary RL and linear MDPs in non-stationary RL. Moreover, when the non-stationarity is relatively small compared to a universal gap (which still allows a certain non-stationarity), we show that the DBE dimension is the same as the BE dimension of one MDP instance of the non-stationary MDPs. 

\noindent{\bf Algorithm:} We then design a new confidence-set based algorithm~\alg for non-stationary MDP, by greedily selecting the candidate value function in the confidence region. This is in contrast to the UCB-type algorithms adopted by all previous studies of non-stationary MDPs. In fact, a UCB-type algorithm is not easily applicable to non-stationary MDPs with general function approximation due to the difficulty of finding an appropriate bonus term. Our main design novelty lies in the construction of the confidence region, which features the sliding window mechanism, and incorporates local variation budget in order to exactly capture the distribution mismatch between current episode and all episodes in the sliding window. Such a design ensures the optimal state-action value function in current episode to lie within the confidence region, and hence the optimism principle remains valid.

\noindent{\bf Theory:} We theoretically characterize the dynamic regret of~\alg in Theorem~\ref{thm:DR}. To demonstrate the advantage of \alg, we compare our regret bound of \alg to that of previously proposed UCB-type algorithms ~\citep{Huozhi_Zhou:adversarial_MDP:2022} for non-stationary linear and tabular MDPs. The comparison shows that our confidence-set based algorithm performs better in terms of the linear feature dimension $\tilde d$ and the horizon $H$, where the dependency on $H$ also matches the minimax lower bound given in~\citet{Huozhi_Zhou:adversarial_MDP:2022}. Our bound is slightly worse in the average variation budget, which suggests that our algorithm is advantageous over UCB-type algorithms in the small variation scenario. 

\noindent{\bf Analysis:} Technically, our analysis features a few new developments. (a) We develop a distribution shift lemma to handle transition kernel variations over time. (b) We come up with new auxiliary random variables to form appropriate martingale differences and obtain the concentration results. (c) We use an auxiliary MDP to help bound the difference of two expectations under different underlying models.

\subsection{Related Work}
\noindent{\bf Non-stationary tabular MDPs:} Most works on non-stationary tabular MDPs considered static regret~\cite{Peter_auer:adversarial_MDP:NeurIPS:2008,Pratik_Gajane:adversarial_MDP_arXiv:2018,Even-Dar:MDP:2009,Yu_Jia:MDP:2009,Jia_Yuanyu:MDP:2009,Neu:MDP:MDP:2010,Neu:MDP:MDP:2012,Zimin:MDP:2013,Dekel:MDP:2013,Rosenberg:MDP:2019,Chi_Jin:MDP:2020}. A few recent studies~\cite{Cheung:adversarial_MDP:PMLR:2020,Yingjie_Fei:adversarial_MDP:NeurIPS:2020,Weichao_Mao:adversarial_MDP:2021} focused on dynamic regret for non-stationary tabular MDPs. Specifically, assuming time-varying transitions and rewards, \citet{Cheung:adversarial_MDP:PMLR:2020} proposed a sliding window approach, and \citet{Weichao_Mao:adversarial_MDP:2021} used restart mechanism to handle non-stationarity. While the first two works adopted value-based algorithms, 
\citet{Yingjie_Fei:adversarial_MDP:NeurIPS:2020} applied a policy optimization algorithm for full-information feedback of rewards and time-invariant transitions. 

\noindent{\bf Non-stationary MDPs with function approximation:} Under non-stationary MDPs with continuous environment where the state-action forms a metric space, \citet{Domingues:adversarial_MDP:2020} proposed a kernel-based algorithm. Two concurrent works \citet{Huozhi_Zhou:adversarial_MDP:2022} and \citet{Ahmed_Touati:adversarial_MDP:2020} considered non-stationary RL under linear MDPs, where \citet{Huozhi_Zhou:adversarial_MDP:2022} considered dynamic regret and \citet{Ahmed_Touati:adversarial_MDP:2020} studied static regret. To handle non-stationarity, \citet{Huozhi_Zhou:adversarial_MDP:2022} adopted a scheme of restarting the base LSVI-UCB algorithm while \citet{Ahmed_Touati:adversarial_MDP:2020} used weighted least squares value iteration with exponential weights on past data. Under the non-stationary MDPs with linear mixture function approximation of both transitions and rewards, \citet{Han_Zhong:adversarial_MDP:2021} 
considered bandit feedback rewards and dynamic regret. Moreover, \citet{Chenyu_Wei:MDP:2021} proposed black-box reduction approach that converts algorithm with optimal regret in stationary MDPs into another algorithm for non-stationary MDPs. 

Recently, \citet{Foster:adversarial_DEC:2022} generalized the decision-estimation coefficient (DEC) framework to non-stationary RL setting with the goal of minimizing the static regret. Their framework can potentially cover majority problems but the connection between their result and the existing results is still not well understood. We also remark that the performance under the DEC framework is often worse than the best-known result when restricted to special cases. Further, their work focused on the static regret, whereas our work potentially maintains the sharp performance when restricting to special cases, and our performance metric of dynamic regret is more general. 

\noindent{\bf Static MDP with general function approximation:} Broadly speaking, the line of research on designing sample-efficient RL algorithms with general function approximations in the past has been mainly focused on the static RL setting. \citet{Russo:Eluder:2013,Osband:Eluder:2014} initiated the study on the minimal structural assumptions that render sample-efficient learning by proposing a structural condition called Eluder dimension, and  \citet{Ruosong_Wang:eluder:2020} then extended LSVI-UCB for general function approximation with small Eluder dimension. Another well-studied direction is the low-rank conditions, including Bellman rank~\citep{Kefan_Dong:Bellman_rank:2019,Nan_Jiang:Bellman_rank:2017} for model-free setting and witness rank~\citep{Wen_Sun:witness_rank:2019} for model-based setting. \citet{Chi_Jin:bellman_eluder:2020} proposed a complexity named Bellman Eluder (BE) dimension for model-free setting, which subsumes low Bellman rank and low Eluder dimension as special cases. \citet{Simon_Du:bilinear:2021} proposed Bilinear class, which unifies both model-based and model-free RL for a broad class of loss estimators including Bellman error. Sharing the same spirit of unifying model-free and model-based RL, \citet{Foster:DEC:2021} proposed DEC, which is a necessary and sufficient condition for sample-efficient learning, and then they extended it to an adversarial decision making problem with static regret in~\citet{Foster:adversarial_DEC:2022}. While the sample complexity of \citet{Foster:DEC:2021, Simon_Du:bilinear:2021} is generally worse than the best-known result when restricted to special cases, \citet{Zixiang_Chen:ABC:2022} recently extended BE dimension and proposed an Admissible Bellman Characterization (ABC) framework to include both model-free and model-based RL while maintaining sharp sample efficiency. Very recently, \citet{yin2022offline,zhang2022off} consider parametric differentiable function approximation in offline RL, but there is no study in the online regime.

\noindent{\bf Non-stationary bandits:} Broadly speaking, our work is also related to a line of research on non-stationary bandits. Methods have been proposed to handle non-stationarity for various non-stationary multi-armed bandit (MAB) settings, including decaying memorey and sliding windows~\citep{Garivier:bandit:2011,Keskin:MAB:2017} and restart mechanism~\citep{Auer:MAB:2002,Besbes:MAB:2014,Bebes:MAB:2017}, which are widely employed in non-stationary RL. More recently, several works developed methods for unknown variation budget~\citep{Karnin:MAB:2017,Cheung:MAB:2022}, and abrupt changes~\citep{Auer:MAB:2019}. Another line of works focused on Markovian bandits~\cite{Ma:bandit:2018}, non-stationary contextual bandits~\cite{Luo2017EfficientCB,pmlr-v99-chen19b}, linear bandits~\cite{Cheung:adversarial_MDP:PMLR:2019,pmlr-v108-zhao20a}, and bandits with slowly changing rewards~\cite{Besbes:bandit:2019}.

\section{Preliminaries}\label{sec:pre}
\subsection{Non-stationary MDPs}
We consider an episodic MDP with time-varying transitions and rewards $(\mathcal{S},\mathcal{A},H,P,r,x_1)$, where $\mathcal{S}$ is the state space, $\mathcal{A}$ is the action space, $H$ is the length of each episode, $P=\{P_h^k\}_{(k,h)\in[K]\times[H-1]}$ is the collection of non-stationary transition kernels with $P_h^k:\mathcal{S}\times\mathcal{A}\mapsto\triangle(\mathcal{S})$, $r=\{r_h^k\}_{(k,h)\in[K]\times[H]}$ is the collection of adversarial deterministic reward functions with $r_h^k:\mathcal{S}\times\mathcal{A}\mapsto[0,1]$, and $x_1$ is the fixed initial state. 

Suppose an agent sequentially interacts with the non-stationary MDP $(\mathcal{S},\mathcal{A},H,P,r,x_1)$. At the beginning of the $k$-th episode, the reward $\{r_h^k\}_{h\in[H]}$ are adversarially chosen by the environment, which possibly depends on the $(k-1)$ historical trajectories. Meanwhile, the agent determines a policy $\pi^k=\{\pi_h^k\}_{h\in[H]}$ where $\pi_h^k:\mathcal{S}\mapsto\triangle(\mathcal{A})$. At the $h$-th step, the agent observes the state $x_h^k$, takes an action following $a_h^k\sim\pi_h^k(\cdot|x_h^k)$, obtains the reward function $r_h^k$ which determines the received reward $r_h^k(x_h^k,a_h^k)$, and the MDP evolves into the next state $x_{h+1}^k\sim P_h^k(\cdot|x_h^k,a_h^k)$. The $k$-th episode ends after receiving the last reward $r_H^k(x_H^k,a_H^k)$. For convenience, let $x_{H+1}$ be a dummy state and $P_H^k(x_{H+1}|x_H^k,a_H^k)=1$ for any $(x_H^k,a_H^k)\in\mathcal{S}\times\mathcal{A}$. Define the state and state-action value functions of policy $\pi=\{\pi_h\}_{h\in[H]}$ recursively via the following equation
\begin{align*}
    &Q_{h;\stark}^{\pi}(x,a)=r_h^k(x,a)+(\Pb_h V_{h+1;\stark}^{\pi})(x,a), \\
    &V_{h;\stark}^{\pi}(x)=\langle Q_{h;\stark}^{\pi}(x,\cdot),\pi_h(\cdot|x)\rangle_{\mathcal{A}},\enspace V_{H+1;\stark}^{}=0,
\end{align*}
where $\Pb_h$ is the operator defined as $(\Pb_h f)(x,a):=\Eb\left[f(x')|x'\sim P_h(x'|x,a)\right]$ for any function $f:\mathcal{S}\mapsto \Rb$. Here $\langle \cdot,\cdot\rangle_\mathcal{A}$ denotes the inner product over action space $\mathcal{A}$ and the subscript $\mathcal{A}$ is omitted when appropriate.

The performance is measured by the dynamic regret, which quantifies the performance difference between the learning policy and the benchmark policy $\{\pi^{\stark}\}_{k\in[K]}$ where $\pi^\stark=\arg\max_{\pi}V_{1;\stark}^{\pi}(x_1)$. Specifically, the dynamic regret for $K$ episodes is defined as
\begin{align*}
    \textstyle\DR(K):=\sum_{k=1}^{K}\left(V_{1;\stark}^{\pi^{\stark}}-V_{1;\stark}^{\pi^k}\right)(x_1).
\end{align*}

\subsection{Function Approximation}
Consider a function class $\Fc=\Fc_1\times\cdots\times\Fc_H$, where $\Fc_h\subseteq (\mathcal{S}\times\mathcal{A}\mapsto[0,H-h+1])$ offers a collection of candidate functions to approximate $Q_{h;\stark}^{\pi_{\stark}}$, denoted as $Q_{h;\stark}^*$. Since each episode ends in $H$ steps, we set $f_{H+1}=0$. We make the following standard assumptions on the function class $\Fc$.
\begin{assumption}[Realizability]\label{assump:realizability}
$Q_{h;\stark}^{*}\in\Fc_h$ for all $(k,h)\in[K]\times[H]$.
\end{assumption}
Realizability assumption requires that the optimal state-action value function in each episode is contained in the function class $\Fc$ with no approximation error, i.e., $(Q_{1;\stark}^*,Q_{2;\stark}^*,\cdots,Q_{H;\stark}^*)\in\Fc$ for $k\in[K]$.

Given functions $f=(f_1,f_2,\cdots,f_H)$ where $f_h\in (\mathcal{S}\times\mathcal{A}\mapsto[0,H-h+1])$, define
\begin{align*}
    &(\Tc_h^k f_{h+1})(x,a):=r_h^k(x,a)+(\Pb_h^k f_{h+1})(x,a), \\
    &(\Pb_h^k f_{h+1})(x,a)=\Eb_{x'\sim P_h^k(\cdot|x,a)}[\max_{a'\in\mathcal{A}}f_{h+1}(x',a')], \quad
\end{align*}
where $\Tc_h^k$ is the Bellman operator at step $h$ in episode $k$. Note that $Q_{h;\stark}^*(x,a)=(\Tc_h^kQ_{h+1;\stark}^*)(x,a)$ for all valid $x,a,h$. Moreover, we define $\Tc_h^k\Fc_{h+1}=\{\Tc_h^kf_{h+1}:f_{h+1}\in\Fc_{h+1}\}$.
\begin{assumption}[Completeness]\label{assump:completeness}
$\Tc_h^k\Fc_{h+1}\subseteq \Fc_h$ for all $(k,h)\in[K]\times[H]$.
\end{assumption}
Completeness assumption requires the function class $\Fc$ is closed under Bellman operators of any episode.

\section{Dynamic Eluder Dimension}\label{sec:DBE}
In this section, we introduce a new complexity measure for non-stationary MDPs. We start with the following $\epsilon$-independence between distributions and the distributional Eluder dimension.
\begin{definition}[$\epsilon$-independence between distributions~\citep{Chi_Jin:bellman_eluder:2020}]\label{defn:indpt}
Let $\Gc$ be a function class defined on $\mathcal{X}$, and $\nu,\mu_1,\mu_2,\cdots,\mu_n$ be probability measures over $\mathcal{X}$. We say $\nu$ is $\epsilon$-independent of $\{\mu_1,\mu_2.\cdots,\mu_n\}$ with respect to $\Gc$ if there exists $g\in\Gc$ such that $\sum_{i=1}^n\left(\Eb_{x\sim\mu_i}[g(x)]\right)^2\leq \epsilon^2$, but $|\Eb_{x\sim\nu}[g(x)]|>\epsilon$. 
\end{definition}

\begin{definition}[Distributional Eluder (DE) dimension~\citep{Chi_Jin:bellman_eluder:2020}]
Let $\Gc$ be a function class defined on $\mathcal{X}$, and $\Pi$ be a family of probability measures over $\mathcal{X}$. The distributinoal Eluder dimension $\GDE(\Gc,\Pi,\epsilon)$ is the length of the longest sequence $\{\rho_1,\rho_2,\cdots,\rho_n\}\subseteq \Pi$ such that there exists $\epsilon'\geq \epsilon$ where $\rho_i$ is $\epsilon'$-independent of $\{\rho_1,\rho_2,\cdots,\rho_{i-1}\}$ for all $i\in[n].$
\end{definition}
The next definition of Bellman Eluder dimension is first introduced in~\citet{Chi_Jin:bellman_eluder:2020} for stationary MDPs.  
\begin{definition}[Bellman Eluder dimension (BE)]\label{def:BE}
Let $(I-\Tc_h)\Fc:=\{f_h-\Tc_h f_{h+1}:f\in\mathcal{F},k\in[K]\}$ be the set of Bellman residuals in all episodes induced by $\mathcal{F}$ at step $h$, and $\Pi=\{\Pi_h\}_{h\in[H]}$ be a collection of $H$ probability measure families over $\mathcal{S}\times\mathcal{A}$. The $\epsilon$-Bellman Eluder dimension of $\mathcal{F}$ with respect to $\Pi$ is defined as
\begin{align*}
    \BE(\mathcal{F},\Pi,\epsilon):=\max_{h\in[H]}\GDE\left((I-\Tc_h)\mathcal{F},\Pi_h,\epsilon\right).
\end{align*}
\end{definition}
For non-stationary MDPs, the Bellman operators $\Tc_h$ varies over time, and hence we introduce our new complexity measure called dynamic Bellman Eluder dimension for non-stationary MDPs.
\begin{definition}[Dynamic Bellman Eluder (DBE) dimension]\label{def:DBE}
Let $(I-\bar{\Tc}_h)\Fc:=\{f_h-\Tc_h^k f_{h+1}:f\in\mathcal{F},k\in[K]\}$ be the set of Bellman residuals in all episodes induced by $\mathcal{F}$ at step $h$, and $\Pi=\{\Pi_h\}_{h\in[H]}$ be a collection of $H$ probability measure families over $\mathcal{S}\times\mathcal{A}$. The dynamic Bellman Eluder dimension of $\mathcal{F}$ with respect to $\Pi$ is defined as
\begin{align*}
    \GBE(\mathcal{F},\Pi,\epsilon):=\max_{h\in[H]}\GDE\left((I-\bar{\Tc}_h)\mathcal{F},\Pi_h,\epsilon\right).
\end{align*}
\end{definition}
We focus on the following choice of distribution family $\Dc_{\Delta}=\{\Dc_{\Delta,h}\}_{h\in[H]}$ where $\Dc_{\Delta,h}=\{\delta_{(s,a)}:s\in\mathcal{S},a\in\mathcal{A}\}$. However, our result can be adapted to $\Dc_\Fc=\{\Dc_{\Fc,h}\}_{h\in[H]}$ where $\Dc_{\Fc,h}$ denotes the collection of all probability measures over $\Sc\times \Ac$ at $h$-th step, generated by executing the greedy policy $\pi_f$ induced by any $f\in\Fc$.

The DBE dimension is the distributional Eluder dimension on the function class $(I-\bar{\Tc}_h)\Fc$ in all episodes, maximizing over step $h\in[H]$, 
which can be viewed as an extension of BE dimension to non-stationary MDPs. The main difference between DBE dimension and BE dimension is that the Bellman operator $\Tc_h^k$ is time-varying, and we include all the Bellman residues induced by $\Tc_h^k$ for $k\in[K]$ in the function class. In general, the DBE dimension could be substantially larger than the BE dimension due the fact that the class of functions can be significantly larger. However, we can show that, if the variations in both transitions and rewards are relatively small compared to a universal gap $\widetilde{\delta}_{\epsilon}^u$ defined below, then the DBE dimension equals to the BE dimension with respect to one MDP instance of the non-stationary MDPs.

\begin{definition}[Universal gap]\label{defn:gap}
If $\nu$ is $\epsilon$-independent of $\boldsymbol{\mu}=(\mu_1,\ldots,\mu_n)$ with respect to $g\in\Gc$, we define gap $\widetilde{\delta}_{g,\epsilon;\boldsymbol{\mu},\nu}=|\Eb_{x\sim\nu}[g(x)]|-\epsilon$. The universal gap with respect to a function class $\Gc$ is $\widetilde{\delta}_{\epsilon}^u=\inf_{g\in\Gc,\epsilon'\geq \epsilon, \boldsymbol{\mu}_g}\widetilde{\delta}_{g,\epsilon';\boldsymbol{\mu}_g}$ where $\boldsymbol{\mu}_g$ is any $\epsilon'$-independent sequence with respect to $g$.
\end{definition}

\begin{proposition}[Informal]\label{propo:dbe}
If the variations in transitions and rewards are relatively small compared to the universal gap $\widetilde{\delta}_\epsilon^u$ with respect to $(I-\Tc_h^k)\Fc$ for $k\in[2:K]$, then 
\begin{align*}
    &\GBE(\mathcal{F},\Pi,\epsilon)= \max_{h\in[H]}\GDE((I-\Tc_h^1)\Fc,\Pi,\epsilon),
\end{align*}
where the latter is exactly the BE dimension of the first MDP instance of the non-stationary MDPs. 
\end{proposition}
The formal statement of the proposition (see Proposition~\ref{app-prop-1}) and its proof is provided in Section~\ref{app:discuss-DBE}. The intuition is if the variations in transitions and rewards are small (but does not necessarily vanish), then the set of functions $(I-\Tc_h^k)\Fc$ for $k\in[2:K]$ is relatively close to $(I-\Tc_h^1)\Fc$. Therefore their union $(I-\bar{\Tc}_h)\Fc$, constructed for the DBE dimension, remains close to $(I-\Tc_h^1)\Fc$.

Under static MDPs, the DBE dimension naturally reduces to BE dimension, and therefore it subsumes a majority tractable problem classes in stationary RL. Moreover, the DBE framework further includes more tractable problem classes in non-stationary RL. Below we show that our DBE dimension covers non-stationary linear MDPs. 
\begin{definition}[Non-stationary Linear MDPs~\citep{Huozhi_Zhou:adversarial_MDP:2022}]
    For linear MDP with feature map $\phi:\Sc\times\Ac\mapsto \Rb^d$, there exists an unknown measure $\mu_{h}^k$ on $\Sc$ and a vector $\theta_h^k\in\Rb^d$ satisfying $P_h^k(s'|s,a)=\phi(s,a)^\top\mu_{h}^k(s')$ and $r_h^k(s,a)=\phi(s,a)^\top\theta_{h}^k$,
    \if{0}
    \begin{align*}
        &P_h^k(s'|s,a)=\phi(s,a)^\top\mu_{h}^k(s') \\
        &r_h^k(s,a)=\phi(s,a)^\top\theta_{h}^k,
    \end{align*}
    \fi
    where $\norm{\phi(s,a)}\leq 1$ and $\max\{\norm{\mu_h^k},\norm{\theta_h^k}\}\leq \sqrt{d}$ for all $(h,k)\in[H]\times[K]$.
\end{definition}
The next proposition shows that the DBE dimension of non-stationary linear MDPs scales with the linear feature dimension $\widetilde{\Oc}(d)$. The proof is shown in Appendix~\ref{app:eluder-dim-linear}. 
\begin{proposition}\label{prop:linear-DBE}
The DBE dimension of non-stationary linear MDPs with the feature dimension $d$ satisfies
\begin{align*}
    \GBE(\mathcal{F},\Dc_\Fc,\epsilon)\leq \Oc\left(1+d\log\left(16H^2d/\epsilon^2+1\right)\right).
\end{align*}
\end{proposition}

\section{Algorithm}\label{sec:alg}
In this section, we propose our algorithm \alg for non-stationary MDPs with general function approximation. 

At high level, \alg differentiates from the GOLF algorithm~\citep{Chi_Jin:bellman_eluder:2020} for static MDPs with general function approximation in its novel designs to handle the non-stationarity of transition kernels and rewards. Specifically, \alg features the sliding window mechanism and incorporates local variation budget in order to exactly capture the distribution mismatch between current episode and all episodes in the sliding window. Such a design ensures the optimal state-action value function in current episode to lie within the confidence region, and hence the optimism principle remains valid.

Further in the context of the previous studies of non-stationary MDPs, \alg is the first confidence-set based algorithm, to the best of our knowledge. In fact, a UCB-type algorithm is not easily applicable to non-stationary MDPs with general function approximation due to the difficulty of finding an appropriate bonus term. 
As we will show in Section~\ref{sec:guarantee} by an example of non-stationary linear MDPs, \alg performs better than the best known UCB-type algorithms in small variation budget scenarios. 
\begin{algorithm}
\begin{algorithmic}[1]\label{Alg:1}
\caption{\alg (Sliding Window Optimistic-based Exploration and Approximation under non-stationary MDPs)}
\STATE {{\bf Input:} $\Dc_1,\cdots,\Dc_H\leftarrow \varnothing$, $\Bc^0\leftarrow \Fc$.} \label{alg: Offline RL}
\FOR{{\bf episode} $k$ from 1 to $K$}
\STATE{{\bf Choose} $\pi^k=\pi_{f^k}$,\newline where $f^k=\arg\max_{f\in\Bc^{k-1}}f_1(x_1,\pi_f(x_1))$.} \label{alg-line:1}
\STATE{{\bf Collect} a trajectory $(x_1^{k},a_1^{k},\cdots,x_H^{k},a_H^{k},x_{H+1}^{k})$ by following $\pi^k$ and reward function $\{r_h^k\}_{h\in[H]}$.} \label{alg-line:2}
\STATE{{\bf Augment} $\Dc_h\hspace{-0.02in}=\hspace{-0.02in}\Dc_h\hspace{-0.02in}\cup\hspace{-0.02in}\{(x_h^{k},a_h^{k},x_{h+1}^{k})\}$, $\forall h\in[H]$.}
\STATE{{\bf} Update $\Bc^k \hspace{-0.02in}=\hspace{-0.02in}\{f\hspace{-0.02in}\in\hspace{-0.02in}\mathcal{F}\hspace{-0.02in}:\Lc_{\Dc_h}(f_h,f_{h+1})\hspace{-0.02in}\leq \inf_{\hspace{-0.01in}g\in\Gc_h} \hspace{-0.03in}\Lc_{\hspace{-0.01in}\Dc_h}\hspace{-0.02in}(g,\hspace{-0.02in}f_{h+1})\hspace{-0.01in}+\hspace{-0.01in}\beta\hspace{-0.01in}+\hspace{-0.01in}2\hspace{-0.01in}H^2\hspace{-0.02in}\Delta_P^w(k,h)$, $\forall h\hspace{-0.02in}\in\hspace{-0.02in}[H]\}$.} \label{alg-line:3}
\ENDFOR
\end{algorithmic}
\end{algorithm}

\if{0}
To quantify the non-stationarity of the adversarial MDP, we define the variation in rewards of adjacent episodes, variations in transitions
{\tiny
\begin{align}
    &\Delta_P^w(k,h):=\sum_{t=1\lor (k-w)}^{k}\sup_{x\in\mathcal{S},a\in\mathcal{A}}\norm{P_h^{k}(\cdot|x,a)-P_h^{t}(\cdot|x,a)}_1 \\
    &\label{eqn:pathlength-3}
    :=\sum_{t=1\lor (k-w)}^{k}\norm{P_{h}^{k}-P_{h}^{t}}_{1,\infty}.
\end{align}
}Here we assume the local pathlength $\Delta_P^w(k,h)$ is known.
\fi

The pseudocode of \alg is presented in Algorithm~\ref{Alg:1}. \alg initializes the dataset $\{\Dc_h\}_{h\in[H]}$ to be empty sets, and confidence set $\Bc^0$ to be $\Fc$. Then, in each episode, \alg performs the following two steps:

\noindent {\bf Optimistic planning step} (Line 3) 
greedily selects the most optimistic state-action value function $f^k$ from the 
confidence set $\Bc^{k-1}$ constructed in the last episode, and chooses the corresponding greedy policy $\pi_k$ associated with $f^k$.

\noindent {\bf Sliding window squared Bellman error} is defined as 
{
\begin{align}
&\Lc_{\Dc_h}(\xi_h,\zeta_{h+1})= \sum_{t=1\lor (k-w)}^{k}\left(\xi_h(x_h^{t},a_h^{t})-r_h^k(x_h^{t},a_h^{t})-\max_{a'\in\mathcal{A}}\zeta_{h+1}(x_{h+1}^{t},a')\right)^2. \label{eqn:loss}
\end{align}
}Note that in episode $k$, we use the latest reward information $r_h^k$ over the entire window, rather than $r_h^t$, to form the sliding window squared Bellman error. Such construction exploits the most recent information of the reward function $r_h^k$ to maximally reduce the non-stationarity of rewards. 
Therefore, $\Lc_{\Dc_h}$ tends to be small as long as the transition kernel difference between episode $k$ and $t$ is small. Furthermore, we adopt the sliding window in the squared loss (\ref{eqn:loss}), which is based on the ``forgetting principle''~\citep{Garivier:bandit:2011} where the squared loss estimated at episode $k$ relies on the observed history during episode $1\lor(k-w)$ to $k$ instead of all prior observations. The rationale is that under non-stationarity setting, the historical observations far in the past are obsolete, and they are not as informative for the evaluation of the squared loss.

\noindent {\bf Confidence set updating step} (Line 4-6) 
first executes policy $\pi^k$ and collects data for the current episode, and then updates the confidence set based on the new data.

The key novel ingredient of \alg lies in the construction of the confidence set $\Bc^k$. For each $h\in[H]$, \alg maintains a local regression constraint using the collected data $\Dc_h$
\begin{align*}
    &\Lc_{\Dc_h}(f_h,f_{h+1}) \hspace{-0.02in}\leq \hspace{-0.02in}\inf_{g\in\Gc_h}\Lc_{\Dc_h}\hspace{-0.02in}(g,f_{h+1})\hspace{-0.02in}+\hspace{-0.02in}\beta\hspace{-0.02in}+\hspace{-0.02in}2H^2\Delta_P^w(k,h),
\end{align*}
where $\beta$ is a confidence parameter, and $\Delta_P^w$ is the local variation budget defined by
\begin{align}
   \Delta_{P}^{w}(k,h) 
    \hspace{-0.02in}=\hspace{-0.1in}\sum_{t=1\lor (k-w)}^{k}\hspace{-0.02in}\sup_{x\in\mathcal{S},a\in\mathcal{A}}\hspace{-0.02in}\norm{(P_h^{k}-P_h^{t})(\cdot|x,a)}_1 . \label{eqn:pathlength-3}
\end{align}

Since the transition kernel varies across episodes, we include an additional term of the local variation budget $\Delta_P^w(k,h)$ in the definition of $\Bc_k$.
Intuitively, the local variation budget $\Delta_P^w(k,h)$ captures the cumulative transition kernel differences between current episode and all previous episode in the sliding window. Therefore, by compensating a term involving $\Delta_P^w(k,h)$ in the confidence set $\Bc_k$, we ensure that the optimal state-action value function in the $k$-th episode $Q_{h;\stark}^*$ still lies in the confidence set $\Bc^k$ with high probability (see Lemma~\ref{lemma:Qstar}).

We remark that the assumption on the local variation budget involving transition functions are unknown could be relaxed. Inspired by the standard technique to handle unknown variation budget as in linear nonstationary MDPs~\citep{Huozhi_Zhou:adversarial_MDP:2022}, we propose the following modification of the algorithm. We remove the local variation budget in the bonus term in the algorithm, and instead, design a strategy to adapt the window size to the variation budget (without knowing its value) as in the EXP3-P algorithm~\cite{BubeckC12}. It has been shown that as long as the window sizes are picked to densely cover the entire value range of the window size, such a scheme will result in a performance close enough to the case as if the window size is picked in an optimal way when the variation budget is known. We expect that such a scheme will achieve the same regret (in terms of scaling) as the case with the knowledge of the variation budget. We will investigate the feasibility of the proposed strategy and leave the detailed mathematical analysis in the future work.

\section{Theoretical Guarantees}\label{sec:guarantee}

In this section, we first provide our main theoretical result for \algg, and then present a proof sketch that highlights our novel developments in the analysis.

\subsection{Main Results}
In this section, we provide our characterization of the dynamic regret for \algg. 

We first state the following generalized completeness assumption~\citep{Antos:RL:2008,Jinglin_Chen:RL:2019,Chi_Jin:bellman_eluder:2020}. 
Let $\Gc=\Gc_1\times\cdots\times\Gc_H$ be an auxiliary function class provided to the learner where $\Gc_h\subseteq (\mathcal{S}\times\mathcal{A}\mapsto[0,H-h+1])$.

\begin{assumption}[Generalized completeness]\label{assump:completeness-2}
$\Tc_h^k\Fc_{h+1}\subseteq \Gc_h$ for all $(k,h)\in[K]\times[H]$.
\end{assumption}
If we choose $\Gc=\Fc$, then Assumption~\ref{assump:completeness-2} is equivalent to the standard completeness assumption (Assumption~\ref{assump:completeness}). Without loss of generality, we assume $\Fc\subseteq \Gc$ and $\Gc=\Fc\cup\Gc$.

Moreover, to quantify the non-stationarity, we define the variation in rewards of adjacent episodes and the variation in transition kernels of adjacent episodes as
{
\begin{align}
    &\Delta_{R}(K)=\hspace{-0.02in}\sum_{k=1}^{K}\sum_{h=1}^{H}\sup_{x\in\mathcal{S},a\in\mathcal{A}}|(r_h^{k}-r_h^{k-1})(x,a)|,  \label{eqn:pathlength-1} \\
    &\Delta_P(K)=\hspace{-0.02in}\sum_{k=1}^{K}\sum_{h=1}^{H}\sup_{x\in\mathcal{S},a\in\mathcal{A}}\hspace{-0.02in}\norm{(P_{h}^{k}\hspace{-0.02in}-\hspace{-0.02in}P_{h}^{k-1})(\cdot|x,a)}_{1}  \label{eqn:pathlength-2} ,
\end{align}
}where we define $P_h^0=P_h^1$ and $r_h^0=r_h^1$ for all $h\in[H]$. 

The dynamic regret of our algorithm~\alg is characterized in the following theorem.
\begin{theorem}[Dynamic regret of \algg]\label{thm:DR}
Under Assumption~\ref{assump:realizability} and Assumption~\ref{assump:completeness-2}, there exists an absolute constant $c$ such that for any $\delta\in(0,1]$, $K\in\mathbb{N}$, if we choose $\beta=cH^2\log\frac{KH|\Gc|}{\delta}$ in \algg, then with probability at least $1-\delta$, for all $k\in[K]$, when $k\geq \min\{w+1,\GBE(\Fc,\mathcal{D}_{\Delta,h},\sqrt{1/w})\}$ we have
{
\begin{align*}
  &\DR(k) =\Delta_R(k)+H\Delta_P(k) +\mathcal{O}\Bigl(\sqrt{w}  +\frac{H^2k}{\sqrt{w}}\sqrt{d\log[KH|\mathcal{G}|/\delta]} 
    +\frac{H^2k}{\sqrt{w}}\sqrt{d\sup_{t\in[k]}\Delta_{P}^w(t,h)}\Bigl),
\end{align*}
}where $d=\GBE(\Fc,\mathcal{D}_{\Delta,h},\sqrt{1/w})$.
\end{theorem}
Note that the last term depends on the sliding window size $w$, and we can further optimize $w$ if an upper bound of the local variation budget $\Delta_P^w(t,h)$ is given. Below we give an example for optimizing sliding window size $w$. 

Before we proceed, we first define the average variation budget $L$ as
\begin{align}
  L = \max_{h\in[H],t<k} \frac{\sum_{s=t}^{k-1}\sup_{x,a}\|(P_h^{s+1} - P_h^{s})(\cdot|x,a)\|_{1}}{k-t}. \label{eqn:avg-pathlength}
\end{align}
Clearly, we have $L\leq 1$ and $\Delta_P^w(k,h)\leq Lw^2$, and $L$ can be viewed as the the greatest average variation of transition kernels across adjacent episodes over any period of episodes maximized over step $h\in[H]$. Then the following corollary characterizes the dynamic regret by optimizing the window size $w$ based on $L$.
\begin{corollary}\label{coro:thm}
Under the condition of Theorem~\ref{thm:DR} and $|\Gc|>10$, with probability at least $1-\delta$, the following argument holds: if $\sqrt{L}> \frac{1}{K}\left(\sqrt{\log|\Gc|}-\frac{1}{H\sqrt{d}}\right)$, select $w=\lceil\frac{\sqrt{\log|\Gc|}}{\sqrt{L}+\frac{1}{HK\sqrt{d}}}\rceil$ and the dynamic regret is bounded by
{
\begin{align}
    &\widetilde{\Oc}\Bigl(H^{\frac{3}{2}}K^{\frac{1}{2}}d^{\frac{1}{4}}(\log|\Gc|)^{\frac{1}{4}}+H^2Kd^{\frac{1}{2}}L^{\frac{1}{4}}(\log|\Gc|)^{\frac{1}{4}}   +\Delta_R+H\Delta_P\Bigl); \label{eqn:coro-1}
\end{align}
}otherwise, select $w=K$ and the dynamic regret is bounded by $\widetilde{O}\left(H^2K^{\frac{1}{2}}d^{\frac{1}{2}}(\log|\Gc|)^{\frac{1}{2}}\right)$, where $d=\GBE(\Fc,\mathcal{D}_{\Delta,h},\sqrt{1/w})$.
\end{corollary}

We remark that $|\Gc|$ appearing in the $\log$ term can be replaced by its $\epsilon$-covering number $\Nc_{\Gc}(\epsilon)$ to handle the classes with infinite cardinality. In both Theorem~\ref{thm:DR} and Corollary~\ref{coro:thm}, we do not omit $\log|\Gc|$ in $\widetilde{\Oc}$ since for many function classes, $\log |\Gc|$ (or $\log\Nc_{\Gc}(\epsilon)$) can contribute to a polynomial factor. For example, for $\widetilde{d}$ dimensional linear function class, $\log\Nc_{\Gc}(\epsilon)=\widetilde{\Oc}(\widetilde{d})$ where $\widetilde{d}$ is the linear feature dimension. 

Our first term in (\ref{eqn:coro-1}) 
corresponds to the regret of the static MDP 
while the remaining term arises due to the non-stationarity. As a result, when transitions and rewards remain the same over time, our result reduces to $\widetilde{\Oc}\left(H^2K^{\frac{1}{2}}d^{\frac{1}{2}}(\log|\Gc|)^{\frac{1}{2}}\right)$, which matches the static regret of GOLF in \citet{Chi_Jin:bellman_eluder:2020}\footnote{The additional $H$ here is due to the definition of $r_h\in[0,1]$, whereas \citet{Chi_Jin:bellman_eluder:2020} assumes $\sum_h r_h\leq 1$.}.

{\bf Advantage of \algg:} To understand the advantage of \alg over the UCB-based algorithms, we take non-stationary linear MDPs as an example. When specializing to non-stationary linear and tabular MDPs, our result becomes $\widetilde{\Oc}\left(H^{\frac{3}{2}}T^{\frac{1}{2}}\widetilde{d}+HT\widetilde{d}^{\frac{3}{4}}L^{\frac{1}{4}}+TL_{\theta}\right)$ where $T=HK$, $\widetilde{d}$ is the feature dimension for linear MDPs and $\widetilde{d}=|\Sc||\Ac|$ for tabular MDPs, and $L_\theta$ is the average variation budget in rewards. For non-stationary linear MDPs, the result in~\citet{Huozhi_Zhou:adversarial_MDP:2022} is not comparable to ours due to the different definitions of the variation budget of transition kernels. To make a fair comparison, we convert their bound on the dynamic regret\footnote{They consider bandit feedback. By adapting their algorithm and analysis, it turns out that the dynamic regret does not benefit from full information feedback in non-stationary linear MDPs.} to be for tabular MDPs, which gives $\widetilde{\Oc}\left(H^{\frac{3}{2}}T^{\frac{1}{2}}\widetilde{d}^{\frac{3}{2}}+H^{\frac{4}{3}}\widetilde{d}^{\frac{3}{2}}T\widetilde{L}^{\frac{1}{3}}+H^{\frac{4}{3}}\widetilde{d}^{\frac{4}{3}}TL_{\theta}^{\frac{1}{3}}\right)$. The first term corresponds to the regret of static linear MDPs and our result has better dependency on the feature dimension $\widetilde{d}$. For the second term due to the non-stationarity of transition kernels, our bound is better in terms of the horizon $H$ and feature dimension $\widetilde{d}$ while worse in terms of the average variation budget of transitions $L$ (note that $L\leq 1$). For the last term caused by the non-stationary of rewards, our result performs better in the variation budget of rewards, horizon $H$ as well as the feature dimension $\widetilde{d}$. 

It also interesting to compare our result with the minimax dynamic regret lower bound $\Omega\left(H^{\frac{1}{2}}T^{\frac{1}{2}}\widetilde{d}+H^{\frac{1}{3}}T\widetilde{d}^{\frac{2}{3}}\widetilde{L}^{\frac{1}{3}}\right)$ developed in \citet{Huozhi_Zhou:adversarial_MDP:2022}  for linear MDPs with non-stationary transitions. For such a case, our result becomes $\widetilde{\Oc}\left(H^{\frac{3}{2}}T^{\frac{1}{2}}\widetilde{d}+HT\widetilde{d}^{\frac{3}{4}}L^{\frac{1}{4}}\right)$. The first term is the regret under stationary MDPs and the second term arises due to the non-stationarity of transitions. We can see that our first term corresponding to static MDPs matches the lower bound both in terms of $T$ and $\tilde d$, whereas the upper bound in \citet{Huozhi_Zhou:adversarial_MDP:2022} matches the lower bound only in $T$.
For the non-stationarity term, our dependency on $H$  and  $\tilde d$ is closer to the lower bound than that in \citet{Huozhi_Zhou:adversarial_MDP:2022}, whereas our dependency on the variation budget is close but does not match the lower bound. 
Overall, these comparisons suggest that our confidence-set based algorithm performs better than UCB-type algorithms in small variation budget scenario under non-stationary linear MDPs.

When the state-action set forms a metric space, \citet{Domingues:adversarial_MDP:2020} proposed a kernel-based approach in nonstationary RL. Ignoring term regarding static MDPs, their result renders $\widetilde{O}\left(SA^{\frac{1}{2}}H^{\frac{4}{3}}TL^{\frac{1}{3}}+SA^{\frac{1}{2}}H^{\frac{4}{3}}TL_\theta^{\frac{1}{3}}\right)$ regret bound in the tabular case while our result becomes $\widetilde{O}\left((SA)^{\frac{3}{4}}HTL^{\frac{1}{4}}+TL_\theta\right)$. For the first term caused by the non-stationarity of transition kernels, our result has better dependency on step $H$, but is worse in the average variation budget of transitions. For the second term caused by the non-stationarity of rewards, the dependency on the variation budget of rewards, horizon $H$ as well as the cardinality of state and action spaces is improved. The comparison suggests our confidence-set based algorithm is advantageous over the kernel-based algorithm in small variation budget and small action space scenario under non-stationary MDPs.

\subsection{Proof Sketch of \Cref{thm:DR}}\label{subsec:proof-sketch}
In this section, we provide a sketch of the proof for Theorem~\ref{thm:DR} and defer all the details to Appendix~\ref{app:thm-DR}.

The preliminary step is to decompose the dynamic regret of \alg into three terms as follows: 
{
\begin{align}
    \DR(k)\leq &H+\underbrace{\sum_{t=1}^{k}\sum_{h=1}^{H}\underset{(x_h,a_h)\sim(\pi^t,(*,t-1))}{\Eb}[(r_h^{t-1}-r_h^t)(x_h,a_h)]}_{\URM{1}} \nonumber \\   &+\hspace{-0.02in}\underbrace{\sum_{t=1}^{k}\sum_{h=1}^{H}\hspace{-0.02in}\left[\underset{(x_h,a_h)\sim(\pi^t,(*,t-1))}{\Eb}\hspace{-0.05in}-\hspace{-0.05in}\underset{(x_h,a_h)\sim(\pi^t,(*,t))}{\Eb}\right]\hspace{-0.04in}[r_h^{t}(x_h,a_h)]}_{\URM{2}} \nonumber 
    \\
    &+\underbrace{\sum_{t=1}^{k}\left(V_{1;(*,t-1)}^{\pi^{(*,t-1)}}-V_{1;(*,t-1)}^{\pi^t}\right)(x_1)}_{\URM{3}} . \label{eqn:dr-decom}
\end{align}
}Term $\URM{1}$ can be bounded by $\Delta_R(K)$ by the definition of the variation budget of rewards (\ref{eqn:pathlength-1}). In the sequel, we aim to bound $\URM{2}$ in step II and bound $\URM{3}$ in the remaining steps.

{\bf Step I}: We introduce a novel auxiliary MDP to help bound term $\URM{2}$. For a fixed tuple $(k,h)\in[K]\times[H]$, we design an episodic MDP $(\mathcal{S},\mathcal{A},H,P^k,\widetilde{r},x_1)$ with reward $\widetilde{r}_{h'}=r_h^k(x,a)\mathbf{1}\{h'=h\}$ and the corresponding state value function of policy $\{\pi_{h'}\}_{h'\in[H]}$ is defined as 
$\widetilde{V}_{h';,(*,k)}^{\pi}$. Then, we show in Lemma~\ref{lemma:transition_dif} that
{
\begin{align*}
    &\left(\underset{(x_h,a_h)\sim(\pi^k,(*,k-1))}{\Eb}-\underset{(x_h,a_h)\sim(\pi^k,(*,k))}{\Eb}\right) [r_h^{k}(x_h,a_h)] \\
    &=\left[\widetilde{V}_{1;(*,k-1)}^{\pi^k}-\widetilde{V}_{1;(*,k)}^{\pi^k}\right](x_1) \leq\sum_{i=1}^{h-1}\sup_{x,a}\norm{(P_{h}^{k}-P_{h}^{k-1})(\cdot|x,a)}_{1}.
\end{align*}
}Replacing $k$ by $t$, and summing over $t\in[k]$, $h\in[H]$ gives
{
\begin{align*}
    &\URM{2} \leq \sum_{t=1}^{k}\sum_{h=1}^{H}\sup_{x,a}\sum_{i=1}^{h-1}\norm{(P_i^{t-1}-P_i^{t})(\cdot|x,a)}_1 \leq \sum_{h=1}^{H} \left(\sum_{t=1}^{k}\sum_{i=1}^{H}\sup_{x,a}\norm{(P_i^{t-1}-P_i^{t})(\cdot|x,a)}_1\right) 
    \leq H\Delta_P(k).
\end{align*}
}

{\bf Step II}: This step together with the next step establishes important properties to bound term $\URM{3}$ in step IV. 

First, we develop the following crucial probability distribution shift lemma, which will handle the transition kernel variation in non-stationary MDPs.
\begin{lemma}[Probability distribution shift lemma]\label{lemma:change-transition}
Suppose $P$ and $Q$ are two probability distributions of a random variable $x$ and define $f_m=\sup_x |f(x)|$. Then we have
{
\begin{align*}
    & \textstyle \left|\left(\underset{x\sim P}{\Eb}f(x)-C\right)^2-\left(\underset{x\sim Q}{\Eb}f(x)-C\right)^2\right|  \textstyle \leq (2f_m+2|C|)f_m\cdot\TV(P,Q).
\end{align*}
}
\end{lemma} 
The proof can be found in Appendix~\ref{app:aux-lemmas}. 

Next, we show in Lemma~\ref{lemma:Qstar} that $Q_{\stark}^*$, the optimal state-action value function at step $h$, lies in the confidence set $\Bc^k$ for all $k\in[K]$ with high probability. The argument is proved by the martingale concentration and the confidence set we design. Technically, we define
{
\begin{align*}
    &\#_{k,h}(x_h^t,a_h^t) =r_h^{k}(s_h^t,a_h^t)  +\underset{x'\sim P_h^t(\cdot|x_h^t,a_h^t)} {\Eb}\max_{a'\in\mathcal{A}}Q_{h+1;(*,k)}(x',a'), 
\end{align*}
}to form an appropriate martingale difference, which is similar to the $h$-th step Bellman update of the state-action value function in episode $k$ except that the expectation is taken with respect to $P_h^t$ instead of $P_h^k$. By Lemma~\ref{lemma:change-transition}, the cumulative mismatch during the sliding window between $\#_{k,h}(x_h^t,a_h^t)$ and the $h$-step Bellman update of state-action value function in episode $k$ is captured by the local pathlength $\Delta_P^w(k,h)$. Finally, by the design of confidence set $\Bc^k$, we can show that $Q_{\stark}^*\in\Bc^k$.

Given $Q_{\stark}^*\in\Bc^k$ for all $k\in[K]$, the optimistic planning step (Line~\ref{alg-line:1}) guarantees $V_{1;(*,k-1)}^*(x_1)\leq \sup_a f_1^k(x_1,a)$ for every episode $k\in[K]$. Combining the optimism and the generalized policy loss decomposition (see Lemma~\ref{lemma:dr-decom-2}), we have
{
\begin{align}
    &\URM{3}\leq \sum_{t=1}^{k}\left(\max_{a\in\mathcal{A}}f_{1}^t(x_1,a)-V_{1;(*,t-1)}^{\pi^t}(x_1) \right) \leq \sum_{h=1}^{H}\sum_{t=1}^{k}\underset{(x_h,a_h)\sim(\pi^t,(*,t-1))}{\Eb}[(f_h^t-\Tc_h^{t-1}f_{h+1}^t)(x_h,a_h)]. \label{eqn:sketch:1}
\end{align}
}

{\bf Step III}: 
We will show the sharpness of our confidence set $\Bc^k$. Under the construction of $\Bc^k$, $f^k$ selected from $\Bc^{k-1}$ is guaranteed to have small loss $\Lc_{\Dc_h}(f_h^k,f_h^{h+1})$. Note that data used in episode $k$ are collected by executing $\pi^i$ for one episode for all $i\in[1\lor(k-w),k]$, by the martingale concentration and the completeness assumption. We can show in Lemma~\ref{lemma:concentration} that with high probability, for all $(k,h)\in[K]\times[H]$,

{
\begin{align}
    &\sum_{t=1\lor (k-w-1)}^{k-1}\Bigl[f_h^{k}(s_h^t,a_h^t)-r_h^{k-1}(s_h^t,a_h^t) -\underset{x'\sim P_h^{k-1}(x_h^t,a_h^t)}{\Eb}\max_{a'\in\mathcal{A}}f_{h+1}^k(s',a')\Bigl]^2 \leq 6H^2\Delta_P^w(k-1,h)+\Oc(\beta). \label{eqn:sketch:2}
\end{align}
}Technically, we define the following helpful random variable
{
\begin{align*}
    &\#_{k,h}^f(x_h^t,a_h^t) = r_h^{k}(s_h^t,a_h^t)+\underset{x'\sim P_h^t(x_h^t,a_h^t)}{\Eb}\max_{a'\in\mathcal{A}}f_{h+1}(s',a')
\end{align*}
}to form an appropriate martingale and obtain the martingale concentration result. Then, applying our probability distribution shift lemma (Lemma~\ref{lemma:change-transition}), the definition of $\Bc^k$ and the completeness assumption gives (\ref{eqn:sketch:2}).

{\bf Step IV}: We establish the relationship between (\ref{eqn:sketch:1}) and (\ref{eqn:sketch:2}). Specifically, we aim to upper bound (\ref{eqn:sketch:1}) given (\ref{eqn:sketch:2}) holds. Note that their forms are very similar except that the latter is the squared Bellman error, and the data $(s_t,a_t)$ is taken under policy $\pi^i$ for $i\in[1\lor (k-w):k-1]$. It turns out that the dynamic Bellman Eluder dimension plays an important role in connecting these two terms, as summarized in the following lemma.
\begin{lemma}\label{lemma:GDE}
Given a function class $\Phi$ defined on $\mathcal{X}$ with $|\phi(x)|\leq C$ for all $(g,x)\in\Phi\times\mathcal{X}$, and a family of probability measures $\Pi$ over $\mathcal{X}$. Suppose $\{\phi_k\}_{k\in[K]}\subseteq \Phi$ and $\{\mu_k\}_{k\in[K]}\subseteq \Pi$ satisfy that for all $k\in[K]$, $\sum_{t=1\lor(k-w-1)}^{k-1}(\Eb_{x\sim\mu_t}[\phi_k(x)])^2\leq \beta$. Then for all $k\in[K]$ and $\omega>0$,
{
\begin{align*}
    &\sum_{t=1\lor (k-w)}^{k}\hspace{-0.08in}|\Eb_{x\sim\mu_t}[\phi_t(x)]| \\
    &\leq \mathcal{O}\Bigl(\sqrt{\GDE(\Phi,\Pi,\theta)\beta [k\land (w+1)]}  +\min\{w+1,k,\GDE(\Phi,\Pi,\theta)\}C+[k\land (w+1)]\theta\Bigl).
\end{align*}
}
\end{lemma}
The proof is adapted from the proof of Lemma~41 in \citep{Chi_Jin:bellman_eluder:2020} for the sliding window scenario and provided in Appendix~\ref{app:GDE}.  

Based on the DBE dimension and Lemma~\ref{lemma:GDE}, we are ready to bound $\URM{3}$ via term (\ref{eqn:sketch:1}) . By choosing $\Phi$ to be the function class of Bellman residuals, and $\mu_k$ to be the distribution under policy $\pi^k$, term $\URM{3}$ is upper bounded by
{
\begin{align*}
    & \sum_{h=1}^{H}\sum_{t=1}^{k}\underset{(x_h,a_h)\sim(\pi^t,(*,t-1))}{\Eb}[(f_h^t-\Tc_h^{t-1}f_{h+1}^t)(x_h,a_h)] \\
    & \leq \mathcal{O}\Bigl(\frac{H^2k}{\sqrt{w}}\sqrt{\GBE(\Fc,\mathcal{D}_{\Delta},\sqrt{1/K})\log[KH|\mathcal{F}|/\delta]}  +\hspace{-0.02in}\frac{Hk}{\sqrt{w}}\hspace{-0.02in}\sqrt{\GBE(\Fc,\mathcal{D}_{\Delta},\sqrt{1/K})}\hspace{-0.02in}\sum_{h=1}^{H}\hspace{-0.04in}\sqrt{\hspace{-0.02in}\sup_{k\in[K]}\hspace{-0.04in}\Delta_P^w(k,h)}\Bigl).
\end{align*}
}

Combining all the steps, the dynamic regret of our algorithm~\alg is 
{
\begin{align*}
    &\DR(k)\leq \Delta_R(k)+H\Delta_P(k)+ \mathcal{O}\Bigl(H\sqrt{w} + \frac{H^2k}{\sqrt{w}}\sqrt{d\log[KH|\mathcal{G}|/\delta]} +\frac{H^2k}{\sqrt{w}}\sqrt{d\sup_{t\in[k]}\Delta_{P}^w(t,h)} \Bigl) 
\end{align*}
}where we suppress the first term $H$ in (\ref{eqn:dr-decom}) since it is dominated by the fourth term herein.

\subsection{Bandit Feedback}\label{sec:extension}
In this section, we extend our algorithm to bandit feedback scenario. We defer all the details to Appendix~\ref{app:bandit}.

In bandit feedback scenario, the reward function $r_h^k(\cdot,\cdot)$ is no long available, and the agent can only get access to the reward obtained from the trajectory. Therefore, we need to capture the non-stationarity of rewards in the construction of the sliding window Bellman error and the confidence set. Specifically, we replace the sliding window squared Bellman error (\ref{eqn:loss}) with 
{
\begin{align}
&\Lc_{\Dc_h}(\xi_h,\zeta_{h+1})= \sum_{t=1\lor (k-w)}^{k}\left(\xi_h(x_h^{t},a_h^{t})-r_h^t -\max_{a'\in\mathcal{A}}\zeta_{h+1}(x_{h+1}^{t},a')\right)^2, \nonumber 
\end{align}
}where $r_h^t$ is the reward obtained at step $h$ in episode $t$. Moreover, the local regression constraint for the confidence set is
{
\begin{align*}
&\Lc_{\Dc_h}(f_h,f_{h+1})\leq \inf_{g\in\Gc_h} \Lc_{\Dc_h}(g,f_{h+1})+\beta  +2H^2\Delta_P^w(k,h)+2H\Delta_R^w(k,h),
\end{align*}
}where $\beta$ is a confidence parameter, $\Delta_P^w$ is the local variation budget in transitions defined in (\ref{eqn:pathlength-3}) and $\Delta_R^w$ is the local variation budget in rewards defined as
{
\begin{align*}
 \Delta_{P}^{w}(k,h) 
    =\sum_{t=1\lor (k-w)}^{k}\sup_{x\in\mathcal{S},a\in\mathcal{A}}|(r_h^{k}-r_h^{t})(x,a)| .
\end{align*}
}

Our main theoretical result for the bandit feedback scenario is provided in the next theorem.
\begin{theorem}\label{thm:DR-2}
Under Assumption~\ref{assump:realizability} and Assumption~\ref{assump:completeness-2}, there exists an absolute constant $c$ such that for any $\delta\in(0,1]$, $K\in\mathbb{N}$, if we choose $\beta=cH^2\log\frac{KH|\Gc|}{\delta}$ in \algg, then with probability at least $1-\delta$, for all $k\in[K]$, when $k\geq \min\{w+1,\GBE(\Fc,\mathcal{D}_{\Delta,h},\sqrt{1/w})\}$ we have
{
\begin{align*}
\DR(k) =&\Delta_R(k)+H\Delta_P(k) \\
& +\mathcal{O}\left(\sqrt{w} +\frac{H^2k}{\sqrt{w}}\sqrt{d\log[KH|\mathcal{G}|/\delta]} 
    +\frac{H^2k}{\sqrt{w}}\sqrt{d\sup_{t\in[k]}\Delta_{P}^w(t,h)} +\frac{H^{3/2}k}{\sqrt{w}}\sqrt{d\sup_{t\in[k]}\Delta_{R}^w(t,h)}\right),
\end{align*}
}where $d=\GBE(\Fc,\mathcal{D}_{\Delta,h},\sqrt{1/w})$.
\end{theorem}
Besides the average variation budget $L$ in transitions defined in (\ref{eqn:avg-pathlength}), we define the average variation budget $L_\theta$ in rewards
\begin{align}
  L_\theta = \max_{h\in[H],t<k} \frac{\sum_{s=t}^{k-1}\sup_{x,a}|(r_h^{s+1} - r_h^{s})(x,a)|}{k-t}.  \label{eqn:avg-pathlength-rewards}
\end{align}
By optimizing the window size $w$, we have the following corollary.
\begin{corollary}\label{coro:thm-bandit_fb}
Under the condition of Theorem~\ref{thm:DR-2} and $|\Gc|>10$, with probability at least $1-\delta$, the following argument holds: if $\sqrt{L}+\frac{\sqrt{L_\theta}}{\sqrt{H}}>\frac{1}{K}\left(\sqrt{\log|\Gc|}-\frac{1}{H\sqrt{d}}\right), \text{ select }  w=\lceil\frac{\sqrt{\log|\Gc|}}{\sqrt{L}+\frac{\sqrt{L_\theta}}{\sqrt{H}}+\frac{1}{HK\sqrt{d}}}\rceil$, the dynamic regret is upper-bounded by
{
\begin{align}
    &\widetilde{\Oc}\Bigl(H^{\frac{3}{2}}K^{\frac{1}{2}}d^{\frac{1}{4}}(\log|\Gc|)^{\frac{1}{4}}+H^2KL^{\frac{1}{4}}d^{\frac{1}{2}}(\log|\Gc|)^{\frac{1}{4}} +H^{\frac{7}{4}}KL_\theta^{\frac{1}{4}}d^{\frac{1}{2}}(\log|\Gc|)^{\frac{1}{4}}  +\Delta_R+H\Delta_P\Bigl); \nonumber 
\end{align}
}otherwise, select $w=K$ and the dynamic regret is upper-bounded by $\widetilde{O}\left(H^2K^{\frac{1}{2}}d^{\frac{1}{2}}(\log|\Gc|)^{\frac{1}{2}}\right)$, where $d=\GBE(\Fc,\mathcal{D}_{\Delta,h},\sqrt{1/w})$.
\end{corollary}

\section{Conclusion and Future Work}\label{sec:conclusion}
In this paper, we proposed a new complexity metric named Dynamic Bellman Eluder (DBE) dimension for non-stationary MDPs, which extends the Bellman Eluder (BE) dimension for static MDPs. When the variations in transition kernels and rewards are relatively small compared to a universal gap, we show that the DBE dimension is exactly the BE dimension of one MDP instance in the non-stationary MDPs. 
We then incorporated the sliding window mechanism and a novel design for the confidence set into our confidence-set based algorithm~\algg, and provided its theoretical upper bound on the dynamic regret. We further demonstrate the advantage of our algorithm by comparing our dynamic regret bound to that of previously proposed algorithms for non-stationary linear and tabular MDPs.
One interesting future direction is to further improve the dependency of the dynamic regret on the average variation $L$.

\section*{Acknowledgements}
The work of Y. Liang was supported in part by the U.S. National Science Foundation under the grants DMS-2134145 and RINGS-2148253. The work of R. Huang and J. Yang was supported by the U.S. National Science Foundation under the grants CNS-1956276 and CNS-2003131. M. Yin and Y. Wang were partially supported by National Science Foundation grants \#2007117 and  \#2003257.

\bibliography{example_paper,Ref.bib}
\bibliographystyle{icml2023}

\newpage
\appendix
\onecolumn

\clearpage
\begin{center}
	{\LARGE \textbf{Appendix}}
\end{center}

\section{Proof of~\Cref{propo:dbe}}\label{app:discuss-DBE}
In this section, we extend Bellman Eluder (BE) dimension to dynamic Bellman Eluder dimension (DBE) under the setting of small variations in transitions and rewards in the following steps.

The {\bf First Step} is to generalize the class of Bellman residues considered in Bellman Eluder dimension. We restate the definition of Bellman Eluder dimension~\citep{Chi_Jin:bellman_eluder:2020}.
\begin{definition}[Bellman Eluder dimension (BE)]
Let $(I-\Tc_h)\Fc:=\{f_h-\Tc_h f_{h+1}:f\in\mathcal{F}\}$ be the set of Bellman residuals in all episodes induced by $\mathcal{F}$ at step $h$, and $\Pi=\{\Pi_h\}_{h\in[H]}$ be a collection of $H$ probability measure families over $\mathcal{S}\times\mathcal{A}$. The $\epsilon$-Bellman Eluder dimension of $\mathcal{F}$ with respect to $\Pi$ is defines as
\begin{align*}
    \BE(\mathcal{F},\Pi,\epsilon):=\max_{h\in[H]}\GDE\left((I-\Tc_h)\mathcal{F},\Pi_h,\epsilon\right).
\end{align*}
\end{definition}

For ease of presentation, we use $(f,\Tc_h)$ to denote the element $f_h-\Tc_hf_{h+1}$ in the set $(I-\Tc_h)\Fc$. 
For any $(f,\Tc_h)$ pair, 
we introduce the complement of $(f,\Tc_h)$, denoted  by $(-f,-\Tc_h)$, where $-\Tc_h f' = -r_h + P_h f'$ for any $f'$.  Let $-(I-\Tc_h)\Fc$ be the set of all complements of $(f,\Tc_h)\in (I-\Tc_h)\Fc$. Then, we define the extended class of Bellman residuals
\begin{align*}
    \BE(\widetilde{\Fc},\Pi,\epsilon):=\max_{h\in[H]}\GDE\left((I-\widetilde{\Tc}_h)\widetilde{\mathcal{F}},\Pi_h,\epsilon\right),
\end{align*}
where $(I-\widetilde{\Tc}_h)\widetilde{\mathcal{F}}=((I-\Tc_h)\mathcal{F})\cup (-(I-\Tc_h)\mathcal{F})$. 

We first show that the BE dimension of the extended class of Bellman residuals equals to that of the original class of Bellman residuals, as formalized in the following lemma.
\begin{lemma}\label{lemma:discuss-1}
Let $\tilde{\Fc}$ be defined in the above context, then we have $\BE(\widetilde{\Fc},\Pi,\epsilon)=\BE(\mathcal{F},\Pi,\epsilon)$.
\end{lemma}
\begin{proof}
    Since $(I-\Tc_h)\mathcal{F}\subseteq (I-\widetilde{\Tc}_h)\widetilde{\mathcal{F}}$, it is obvious that $\BE(\widetilde{\Fc},\Pi,\epsilon)\geq \BE(\mathcal{F},\Pi,\epsilon)$. Next, we show $\BE(\widetilde{\Fc},\Pi,\epsilon)\leq \BE(\mathcal{F},\Pi,\epsilon)$.  
    Let $\mu$ be independent of $\rho_1,\ldots,\rho_m$ with respect to $(I-\widetilde{\Tc}_h)\widetilde{\mathcal{F}}$. We aim to show $\mu$ is also independent of $\rho_1,\ldots,\rho_m$ with respect to $(I-\Tc_h)\Fc$. 
    
    By the definition of $\epsilon$-independence between distributions, there exists a function $g$ from either $(I-\Tc_h)\mathcal{F}$ or $-(I-\Tc_h)\mathcal{F}$ such that the there exists $\epsilon'\geq \epsilon$ such that $\sqrt{\sum_{i=1}^m\Eb_{\rho_i}[g]}\leq \epsilon$ and $|\Eb_{\mu}[g]|>\epsilon$. If $g$ is from $(I-\Tc_h)\mathcal{F}$, then $\mu$ is obviously independent of $\rho_1,\ldots,\rho_m$ with respect to $(I-\Tc_h)\Fc$. If $g$ is from $-(I-\Tc_h)\mathcal{F}$, i.e., $g$ has form $g=-f_h-(-\Tc_h)(-f_{h+1})$ for some $f$ and $\Tc_h$, we have
    \begin{align*}
    &\sum_{t=1}^{m}(\Eb_{x\sim\rho_i}[-f_h-(-\Tc_h)(-f_{h+1})])^2=\sum_{t=1}^{m}(\Eb_{x\sim\rho_i}[-f_h+r_h+P_hf_{h+1}])^2\leq \epsilon^2, \\
    &\left|\Eb_{x\sim\mu}[-f_h-(-\Tc_h)(-f_{h+1})]\right|=\left|\Eb_{x\sim\mu}[-f_h+r_h+P_hf_{h+1}]\right|>\epsilon,
\end{align*}
which again implies $\mu$ is independent of $\rho_1,\ldots,\rho_m$ with respect to $\BE(\mathcal{F},\Pi,\epsilon)$.

We have shown that if $\mu$ be independent of $\rho_1,\ldots,\rho_m$ with respect to $\BE(\widetilde{\Fc},\Pi,\epsilon)$, then $\mu$ is also independent of $\rho_1,\ldots,\rho_m$ with respect to $\BE(\mathcal{F},\Pi,\epsilon)$. Therefore, the length of the longest independent sequence in $\BE(\mathcal{F},\Pi,\epsilon)$ must be equal or longer than that in $\BE(\widetilde{\Fc},\Pi,\epsilon)$, i.e., $\BE(\widetilde{\Fc},\Pi,\epsilon)\leq \BE(\mathcal{F},\Pi,\epsilon)$.
\end{proof}

The {\bf Second step} is to investigate the difference between two BE dimensions for different Bellman operators. Before we proceed, we define the gap in the definition of $\epsilon$-independence between distributions.

It turns out that if the variation of the transitions and rewards are smaller than the gap $\widetilde{\delta}_\epsilon^u$, which will be defined later, then two BE dimensions induced by difference Bellman operators are comparable, as summarized in the following theorem.

\begin{lemma}\label{lemma:discuss-2}
Suppose there are two MDP instances with Bellman operator $\Tc^1_h$ and $\Tc^2_h$, where $h\in[H]$. Let $\widetilde{\delta}_{\epsilon}^u$ be the universal gap with respect to $(I-\widetilde{\Tc}_h^2)\widetilde{\mathcal{F}}$ (see Definition~\ref{defn:gap}). Then, if the variation of two instances is relatively small compared to the universal gap $\widetilde{\delta}$ satisfying
\begin{align*}
    \max_{h}\sqrt{6m_{\max}H\left(\sup_{x,a}|r_h^1-r_h^2|+H\cdot\TV(P_h^1,P_h^2)\right)}+\left(\sup_{x,a}|r_h^1-r_h^2|+H\cdot\TV(P_h^1,P_h^2)\right)\leq \widetilde{\delta}_{\epsilon}^u,
\end{align*}
where $m_{\max}=\BE((I-\Tc^2_h)\mathcal{F},\Pi,\epsilon)$, then 
\begin{align*}
    &\GDE((I-\Tc^2_h)\mathcal{F},\Pi,\epsilon)\leq \GDE((I-\Tc_h^1)\mathcal{F},\Pi,\epsilon),
\end{align*}
and
\begin{align*}
    &\GDE(((I-\Tc^2_h)\mathcal{F})\cup((I-\Tc_h^1)\mathcal{F}),\Pi,\epsilon)= \GDE((I-\Tc_h^1)\mathcal{F},\Pi,\epsilon),
\end{align*}
\end{lemma}
\begin{proof}
    Fix $h\in[H]$. Let $\mu_1,\ldots,\mu_m$ be independent sequence with respect to $(I-\widetilde{\Tc}_h^2)\widetilde{\mathcal{F}}$. By the definition of BE dimension, $m\leq \BE((I-\widetilde{\Tc}_h^2)\widetilde{\Fc},\Pi_h,\epsilon)$. If we can show $\mu_1,\ldots,\mu_m$ is also an independent sequence with respect to $(I-\widetilde{\Tc}_h^1)\widetilde{\mathcal{F}}$, then the longest independent sequence with respect to $(I-\widetilde{\Tc}_h^1)\widetilde{\mathcal{F}}$ must be equal or longer than that with respect to $(I-\widetilde{\Tc}_h^2)\widetilde{\mathcal{F}}$ and the proof is complete. In the following, we will focus on proving this argument.

    By the condition, there exists $\epsilon'\geq \epsilon$ such that for all $i\in[m]$ we have
    \begin{align*}
        &\sum_{t=1}^{i-1}(\Eb_{\mu_t}[f_h^i-\Tc_h^2f_{h+1}^i])^2\leq \epsilon'^2, \\
        &|\Eb_{\mu_i}[f_h^i-\Tc_h^2f_{h+1}^i]|\geq \epsilon'+\widetilde{\delta}_{i;\mu_1,\ldots,\mu_i}.
    \end{align*}
    Here, with a little abuse of notation, the subscript $i$ of $\widetilde{\delta}_{i;\mu_1,\ldots,\mu_i}$ represents the function $f_h^i-\Tc_h^2f_{h+1}^i$.
    
    By Lemma~\ref{lemma:discuss:support-1}, we have
    \begin{align}
        \sum_{t=1}^{i-1}(\Eb_{x\sim \mu_t} [f_h^i-\Tc_h^1 f_h^i])^2&\leq \epsilon'^2+6mH\left(\sup_{x,a}|r_h^1-r_h^2|+H\cdot\TV(P_h^1,P_h^2)\right) \nonumber \\
        &\leq \left(\epsilon'+\sqrt{6mH\left(\sup_{x,a}|r_h^1-r_h^2|+H\cdot\TV(P_h^1,P_h^2)\right)}\right)^2. \label{eqn:app-1}
\end{align}
    We point it out that both the $(f^i,\Tc_h^1)$ from $(I-\Tc_h)\Fc$ and $(-f^i,-\Tc_h^i)$ from $-(I-\Tc_h)\Fc$ satisfy the above inequality.

Next, consider 
\begin{align*}
    &\min\bigg\{\bigg| \big|\Eb_{x\sim P} [f_h-r_h^2-P_h^2f_{h+1} ] \big|  -\big| \Eb_{x\sim P} [f_h-r_h^1-P_h^1f_{h+1} ]\big| \bigg|,  \\
    &\qquad \quad \bigg| \big|\Eb_{x\sim P}[-f_h-(-r_h^2)-P_h^2(-f_{h+1})]\big|  -\big|\Eb_{x\sim P}[f_h-r_h^1-P_h^1f_{h+1}]\big|\bigg|\bigg\}.
\end{align*}
The first argument in the $\min$ function corresponds to the difference between pair $(f,\Tc_h^1)$ and $(f,\Tc_h^2)$ while the second one is the difference between pair $(f,\Tc_h^1)$ and $(-f,-\Tc_h^2)$. 

If $(\Eb_{x\sim P}[f_h-r_h^2-P_h^2f_{h+1})(\Eb_{x\sim P}[f_h-r_h^1-P_h^1f_{h+1}])\geq 0$, then by Lemma~\ref{lemma:discuss:support-2}, the first argument in the $\min$ function is upper bounded by $\sup_{x,a}|r_h^1-r_h^2|+\TV(P_h^1,P_h^2)$. 

If $(\Eb_{x\sim P}[f_h-r_h^2-P_h^2f_{h+1})(\Eb_{x\sim P}[f_h-r_h^1-P_h^1f_{h+1}])< 0$, then by Lemma~\ref{lemma:discuss:support-2}, the second argument is upper bounded by $\sup_{x,a}|r_h^1-r_h^2|+H\cdot\TV(P_h^1,P_h^2)$. 

Therefore, the quantity we considered is upper bounded by $\sup_{x,a}|r_h^1-r_h^2|+H\cdot\TV(P_h^1,P_h^2)$. By triangle inequality, 
either
\begin{align}
    \left|\Eb_{x\sim \mu_1}[f_h^i-\Tc_h^1 f_h^i]\right|&\geq \left|\Eb_{x\sim \mu_1}[f_h^i-\Tc_h^2 f_h^i]\right| -\left(\sup_{x,a}|r_h^1-r_h^2|+H\cdot\TV(P_h^1,P_h^2)\right) \nonumber \\
    &\geq\epsilon'+\widetilde{\delta}_{i;\mu_1,\ldots,\mu_i}-\left(\sup_{x,a}|r_h^1-r_h^2|+H\cdot\TV(P_h^1,P_h^2)\right) \label{eqn:app-2}
\end{align}
holds or 
\begin{align}
    \left|\Eb_{x\sim \mu_1}[-f_h^i-(-\Tc_h^1) (-f_h^i)]\right|&\geq \left|\Eb_{x\sim \mu_1}[f_h^i-\Tc_h^2 f_h^i]\right| -\left(\sup_{x,a}|r_h^1-r_h^2|+H\cdot\TV(P_h^1,P_h^2)\right) \nonumber  \\
    &\geq\epsilon'+\widetilde{\delta}_{i;\mu_1,\ldots,\mu_i}-\left(\sup_{x,a}|r_h^1-r_h^2|+H\cdot\TV(P_h^1,P_h^2)\right) \label{eqn:app-3}
\end{align}
holds.

Recall $\widetilde{\delta}_{\epsilon}^u$ is the universal gap with respect to $(I-\widetilde{\Tc}_h^2)\widetilde{\mathcal{F}}$ (see Definition~\ref{defn:gap}), and $m_{\max}=\BE((I-\widetilde{\Tc}_h^2)\widetilde{\Fc},\Pi_h,\epsilon)$. If it holds that
\begin{align*}
    \max_{h}\sqrt{6m_{\max}H\left(\sup_{x,a}|r_h^1-r_h^2|+H\cdot\TV(P_h^1,P_h^2)\right)}+\left(\sup_{x,a}|r_h^1-r_h^2|+H\cdot\TV(P_h^1,P_h^2)\right)\leq \widetilde{\delta}_{\epsilon}^u,
\end{align*}
which implies
\begin{align*}
    \gamma&=\left(\epsilon'+\sqrt{6mH\left(\sup_{x,a}|r_h^1-r_h^2|+H\cdot\TV(P_h^1,P_h^2)\right)}\right) -
    \left(\epsilon'+\widetilde{\delta}_{i;\mu_1,\ldots,\mu_i}-\left(\sup_{x,a}|r_h^1-r_h^2|+H\cdot\TV(P_h^1,P_h^2)\right)\right) \\
    &\geq 0.
\end{align*}
The above inequality together with (\ref{eqn:app-1})-(\ref{eqn:app-3}) shows that for the sequence $\mu_1,\ldots,\mu_i$, there exists a function $g$ from $(I-\widetilde{\Tc}_h^1)\widetilde{\mathcal{F}}$, and a $\widetilde{\epsilon}\in[\epsilon',\epsilon'+\gamma)$ satisfying $\widetilde{\epsilon}\geq \epsilon$ such that 
\begin{align*}
    &\sum_{t=1}^{i-1}(\Eb_{\mu_t}[g])^2\leq \widetilde{\epsilon}^2, \\
    &|\Eb_{\mu_i}[g]|>\widetilde{\epsilon}.
\end{align*}
The above argument holds for all $i\in[m]$ and we conclude that $\mu_1,\ldots,\mu_m$ is again an independent sequence with respect to $(I-\widetilde{\Tc}_h^1)\widetilde{\mathcal{F}}$. The proof for the first inequality is complete by noting that $\BE(\widetilde{\Fc},\Pi,\epsilon)=\BE(\mathcal{F},\Pi,\epsilon)$ by Lemma~\ref{lemma:discuss-1}.

For the second inequality, we are left to show $\GDE(((I-\Tc^2_h)\mathcal{F})\cup((I-\Tc_h^1)\mathcal{F}),\Pi,\epsilon)\leq \GDE((I-\Tc_h^1)\mathcal{F},\Pi,\epsilon)$. The proof is by showcasing every independence sequence with respect to $((I-\Tc^2_h)\mathcal{F})\cup((I-\Tc_h^1)\mathcal{F})$ must also be independent with respect to $(I-\Tc_h^1)\mathcal{F}$, which follows exactly the same argument as above and is omitted here.
\end{proof}

The {\bf Step three} is to build connection between BE dimension to DBE dimension when the variations in transitions and rewards are small.

In general, DBE dimension could be substantially larger than BE dimension of one MDP instance in the non-stationary MDPs. However, if the variation of all instances are small enough compared to the universal gap $\widetilde{\delta}_{k;\epsilon}^u$ with respect to $(I-\Tc_h^k)\Fc$ for all $k\in[2:K]$, DBE dimension is indeed equal to BE dimension. The following proposition is an immediate result from Lemma~\ref{lemma:discuss-2}.
\begin{proposition}\label{app-prop-1}
If it holds that for all $k$,
\begin{align*}
    \max_h\sqrt{6m_k H\left(\sup_{x,a}|r_h^1-r_h^k|+H\cdot\TV(P_h^1,P_h^k)\right)}+\left(\sup_{x,a}|r_h^1-r_h^k|+H\cdot\TV(P_h^1,P_h^k)\right)\leq \widetilde{\delta}_{k;\epsilon}^u,
\end{align*}
where $m_k=\BE((I-\Tc_h^k)\Fc,\Pi,\epsilon)$, and $\widetilde{\delta}_{k;\epsilon}^u$ is the universal gap with respect to function class $(I-\Tc_h^k)\Fc$. Then 
\begin{align*}
        \GBE(\mathcal{F},\Pi,\epsilon)= \GDE((I-\Tc_h^1)\Fc,\Pi,\epsilon),
\end{align*}
where the latter is exactly the BE dimension for the first MDP instance. 
\end{proposition}

\subsection{Supporting Lemmas}
\begin{lemma}\label{lemma:discuss:support-1}
Suppose $f_h\leq H$ for all $h$, and $r,r'\leq 1$, we have
\begin{align*}
    \left|\left(\Eb_{x\sim P}[f_h-r-Pf_{h+1}]\right)^2-\left(\Eb_{x\sim P}[f_h-r'-P'f_{h+1}]\right)^2\right|\leq 6H\left(\sup_{x,a}(r-r')+\TV(P,P')\right).
\end{align*}
\end{lemma}
\begin{proof}
Note that
\begin{align*}
    &\left|\left(\Eb_{x\sim P}[f_h-r-Pf_{h+1}]\right)^2-\left(\Eb_{x\sim P}[f_h-r'-P'f_{h+1}]\right)^2\right| \\
    &\leq 6H\left|\Eb_{x\sim P}[r-r']+\Eb_{x\sim P}[(P-P')f_{h+1}]\right| \\
    &\leq 6H\left(\left|\Eb_{x\sim P}[r-r']\right|+H\left|\Eb_{x\sim P}[(P-P')f_{h+1}]\right|\right) \\
    &\leq6H\left(\sup_{x,a}(r-r')+H\cdot \TV(P,P')\right).
\end{align*}
\end{proof}

\begin{lemma}\label{lemma:discuss:support-2}
Suppose $f_h\leq H$ for all $h$, and $r,r'\leq 1$, we have
\begin{align*}
    \left|\Eb_{x\sim P}[f_h-r-Pf_{h+1}]-\Eb_{x\sim P}[f_h-r'-P'f_{h+1}]\right|\leq \sup_{x,a}(r-r')+H\cdot \TV(P,P').
\end{align*}
\end{lemma}
\begin{proof}
Note that
\begin{align*}
    &\left|\Eb_{x\sim P}[f_h-r-Pf_{h+1}]-\Eb_{x\sim P}[f_h-r'-P'f_{h+1}]\right| \\
    &\leq \left|\Eb_{x\sim P}[r-r']+\Eb_{x\sim P}[(P-P')f_{h+1}]\right| \\
    &\leq \left|\Eb_{x\sim P}[r-r']\right|+\left|\Eb_{x\sim P}[(P-P')f_{h+1}]\right| \\
    &\leq \sup_{x,a}(r-r')+H\cdot \TV(P,P').
\end{align*}
\end{proof}

\section{Proof of Propostion~\ref{prop:linear-DBE}}\label{app:eluder-dim-linear}
In this section, we show the DBE dimension of non-stationary linear MDP is $\widetilde{\Oc}(d)$ where $d$ is the feature dimension.

Define $m=\GBE((I-\Tc_h)\Fc,\Dc_{\Fc,h},\epsilon)$ and let $h=\arg\max_{h\in[H]}\GBE((I-\Tc_h)\Fc,\Dc_{\Fc},\epsilon)$.

Let $\mu_1,\ldots,\mu_m$ be an independent sequence with respect to $(I-\Tc_h)\Fc$. By definition, there exists $(f^1,\Tc^1),\ldots,(f^i,\Tc^i)$ such that for all $i\in[m]$, we have
\begin{align*}
    &\sum_{t=1}^{i-1}\left(\underset{(x,a)\sim\mu_t}{\Eb}\left[(f_h^i-\Tc_h^{i-1}f_{h+1}^i)(x,a)\right]\right)^2\leq \epsilon^2, \quad \mbox{and} \\
    &\left|\underset{(x,a)\sim\mu_i}{\Eb}\left[(f_h^i-\Tc_h^{i-1}f_{h+1}^i)(x,a)\right]\right|>\epsilon.
\end{align*}
We aim to show $m=\widetilde{\Oc}(d)$.

For linear MDP, a natural function class $\Fc_h$ is
\begin{align*}
    \{f\in ((\mathcal{S}\times\mathcal{A})\mapsto [0,H-h+1]):\phi(x,a)^\top w_{h}, \norm{\phi(x,a)}\leq 1, \forall (x,a) \mbox{ and }
\norm{w_{h}}\in 2(H-h+1)\sqrt{d}\}.
\end{align*}
Note that
\begin{align*}
    (f_h^i-\Tc_h^{i-1}f_{h+1}^i)(x,a)=\phi(x,a)^\top(w_{h_i}-\widetilde{w}_{h,i}),
\end{align*}
where $\widetilde{w}_{h,i}=\theta_{h,i-1}+\int_{x'}\mu_{h,i-1}(x')\max_{a}f_{h+1}^i(x',a)$ and we have $\max\{\norm{w_{h,i}},\norm{\widetilde{w}_{h,i}}\}\leq 2H\sqrt{d}$ for all $h\in[H]$.

Therefore, for all $i\in[m]$
\begin{align*}
    &\sum_{t=1}^{i-1}\left(\underset{(x,a)\sim\mu_t}{\Eb}\left[\phi(x,a)^\top(\widetilde{w}_h^i-w_h^{i-1})\right]\right)^2\leq \epsilon^2, \quad \mbox{and} \\
    &\left|\underset{(x,a)\sim\mu_i}{\Eb}\left[\phi(x,a)^\top(\widetilde{w}_h^i-w_h^{i-1})\right]\right|>\epsilon.
\end{align*}

For ease of exposition, we set
\begin{align*}
    &\mathbf{x}_i=\widetilde{w}_h^i-w_h^{i-1}, \quad \mathbf{z}_i=\underset{(x,a)\sim\mu_i}{\Eb}[\phi(x,a)],\quad \mathbf{V}_i=\sum_{t=1}^{i-1}\mathbf{z_t}\mathbf{z_t}^\top+\frac{\epsilon^2}{\zeta}\cdot I, 
\end{align*}
where $\zeta=4H\sqrt{d}$.

The previous argument implies that for all $i\in[m]$,
\begin{align*}
    &\norm{\mathbf{x}_i}_{\mathbf{V}_i}\leq \sqrt{2}\epsilon, \\   
    &\norm{\mathbf{x}_i}_{\mathbf{V}_i}\cdot\norm{\mathbf{z}_i}_{\mathbf{V}_i^{-1}} >\epsilon.
\end{align*}
Therefore, we have $\norm{\mathbf{z}_i}_{\mathbf{V}_i^{-1}}\geq \frac{1}{\sqrt{2}}$.

By matrix determinant lemma,
\begin{align*}
    \det[\mathbf{V}_m]=\det[\mathbf{V}_{m-1}]\left(1+\norm{\mathbf{z}_m}_{\mathbf{V}_m^{-1}}^2\right)\geq \cdots\geq (\frac{3}{2})^{m-1}(\frac{\epsilon^2}{\zeta})^d.
\end{align*}
Moreover,
\begin{align*}
    \det[\mathbf{V}_m]\leq \left(\frac{\trace[\mathbf{V}_m]}{d}\right)^d\leq \left(\frac{\zeta (m-1)}{d}+\frac{\epsilon^2}{\zeta}\right)^d.
\end{align*}
Therefore,
\begin{align*}
    (\frac{3}{2})^{m-1}\leq \left(\frac{\zeta^2(m-1)}{d\epsilon^2}+1\right)^d. \label{eqn:linear-eluder-dim-1}
\end{align*}

Taking logarithm on both sides gives
\begin{align*}
    m\leq 4\left[1+d\log\left(\frac{\zeta^2(m-1)}{d\epsilon^2}+1\right)\right],
\end{align*}
which implies
\begin{align*}
    m\leq \Oc\left(1+d\log\left(\frac{\zeta^2}{\epsilon^2}+1\right)\right).
\end{align*}

\section{Proofs of \algg}\label{app:thm-DR}
In this section, we provide the formal Proof of Theorem~\ref{thm:DR}.

\subsection{Proof of Theorem~\ref{thm:DR}}
We decompose the dynamic regret in the following way
\begin{align*}
    &\DR(k)=\sum_{t=1}^{k}\left(V_{1;(*,t)}^{\pi^{(*,t)}}-V_{1;(*,t)}^{\pi^t}\right)(x_1) \\
    &=\sum_{t=1}^{k}\left(V_{1;(*,t)}^{\pi^{(*,t)}}-V_{1;(*,t-1)}^{\pi^{(*,t-1)}}+V_{1;(*,t-1)}^{\pi^{(*,t-1)}}-V_{1;(*,t-1)}^{\pi^t}+V_{1;(*,t-1)}^{\pi^t}-V_{1;(*,t)}^{\pi^t}\right)(x_1) \\
    &=\left(V_{1;(*,k)}^{\pi^{(*,k)}}-V_{1;(*,0)}^{\pi^{(*,0)}}\right)(x_1)+\sum_{t=1}^{k}\left(V_{1;(*,t-1)}^{\pi^{(*,t-1)}}-V_{1;(*,t-1)}^{\pi^t}\right)(x_1) \\
    &\qquad \qquad \qquad \qquad +\sum_{t=1}^{k}\sum_{h=1}^{H}\left(\underset{(x_h,a_h)\sim(\pi^t,(*,t-1))}{\Eb}[r_h^{t-1}(x_h,a_h)]-\underset{(x_h,a_h)\sim(\pi^t,(*,t))}{\Eb}[r_h^{t}(x_h,a_h)]\right) \\
    &\leq H+\underbrace{\sum_{t=1}^{k}\sum_{h=1}^{H}\left(\underset{(x_h,a_h)\sim(\pi^t,(*,t-1))}{\Eb}[(r_h^{t-1}-r_h^t)(x_h,a_h)]\right)}_{\URM{1}} \\
    &\qquad +\underbrace{\sum_{t=1}^{k}\sum_{h=1}^{H}\left(\left(\underset{(x_h,a_h)\sim(\pi^t,(*,t-1))}{\Eb}-\underset{(x_h,a_h)\sim(\pi^t,(*,t))}{\Eb}\right)[r_h^{t}(x_h,a_h)]\right)}_{\URM{2}} \\
    &\qquad +\underbrace{\sum_{t=1}^{k}\left(V_{1;(*,t-1)}^{\pi^{(*,t-1)}}-V_{1;(*,t-1)}^{\pi^t}\right)(x_1)}_{\URM{3}}.
\end{align*}

By the definition of variation in rewards (\ref{eqn:pathlength-1}), we have $\URM{1}\leq \Delta_R(k)$. 

We bound $\URM{2}$ using the following lemma.
\begin{lemma}\label{lemma:transition_dif}
Fix $(k,h)\in[K]\times[H]$, we have
\begin{align*}
    \left(\underset{(x_h,a_h)\sim(\pi^k,(*,k-1))}{\Eb}-\underset{(x_h,a_h)\sim(\pi^k,(*,k))}{\Eb}\right)[r_h^{k}(x_h,a_h)]
    \leq \sum_{i=1}^{h-1}\norm{\Pb_i^{k-1}-\Pb_i^{k}}_\infty.
\end{align*}
Moreover,
\begin{align*}
    \sum_{t=1}^{k}\sum_{h=1}^{H}\left(\underset{(x_h,a_h)\sim(\pi^t,(*,t-1))}{\Eb}-\underset{(x_h,a_h)\sim(\pi^t,(*,t))}{\Eb}\right)[r_h^{t}(x_h,a_h)]\leq H\Delta_P(k).
\end{align*}
\end{lemma}
The proof of \ref{lemma:transition_dif} is provided in Appendix~\ref{app:lemma:transition_dif}.

Therefore, 
\begin{align*}
    &\DR(k)=H+\Delta_R(k)+H\Delta_P(k)+\underbrace{\sum_{t=1}^{k}\left(V_{1;(*,t-1)}^{\pi^{(*,t-1)}}-V_{1;(*,t-1)}^{\pi^t}\right)(x_1)}_{\URM{3}}.
\end{align*}

Before we proceed, we present the next two lemmas.
\begin{lemma}\label{lemma:Qstar}
If $\beta=cH^2\log\frac{KH|\Gc|}{\delta}$, then with probability at least $1-\delta$, we have $Q_{\stark}^*\in\Bc^{k}$ for all $k\in[K]$.
\end{lemma}
\begin{lemma}\label{lemma:concentration}
If $\beta=cH^2\log\frac{KH|\Gc|}{\delta}$, then with probability at least $1-\delta$, for all $(k,h)\in[K]\times[H]$, we have
\begin{align*}
\sum_{t=1\lor (k-w-1)}^{k-1}\left[f_h^{k}(s_h^t,a_h^t)-r_h^{k-1}(s_h^t,a_h^t) -\underset{x'\sim P_h^{k-1}(x_h^t,a_h^t)}{\Eb}\max_{a'\in\mathcal{A}}f_{h+1}^k(s',a')\right]^2 \leq 6H^2\Delta_P^w(k-1,h)+\Oc(\beta).
\end{align*}
\end{lemma}
The proofs of Lemma~\ref{lemma:Qstar} and \ref{lemma:concentration} are based on martingale concentration and provided in Appendix~\ref{app:concentration}.

By Lemma~\ref{lemma:Qstar}, with probability at least $1-\delta$, we have
\begin{align*}
    &\sum_{k=1}^{K}\left(V_{1;(*,k-1)}^{\pi^{(*,k-1)}}-V_{1;(*,k-1)}^{\pi^k}\right)(x_1) \\
    &\leq \sum_{k=1}^{K}\left(\max_{a\in\mathcal{A}}f_{1}^k(x_1,a)-V_{1;(*,k-1)}^{\pi^k}(x_1) \right)\\
    &\leq \sum_{h=1}^{H}\sum_{k=1}^{K}\underset{(x_h,a_h)\sim(\pi^k,(*,k-1))}{\Eb}[(f_h^k-\Tc_h^{k-1}f_{h+1}^k)(x_h,a_h)],
\end{align*}
where the first inequality follows from Lemma~\ref{lemma:Qstar} and the optimistic planning step (line 3) in Algorithm~\ref{Alg:1} which guarantees that $V_{1;(*,k-1)}^*\leq \sup_{a}f_1^k(x_1,a)$ for every episode $k$, the last inequality follows from generalized policy loss decomposition (Lemma~\ref{lemma:dr-decom-2}) and the fact that $\pi^k=\pi_{f^k}$ (line 3 in Algorithm~\ref{Alg:1}).

The next lemma is adapted from (\cite{Chi_Jin:bellman_eluder:2020}) and the proof can be found in Appendix~\ref{app:GDE}.
\begin{lemma}
Given a function class $\Phi$ defined on $\mathcal{X}$ with $|\phi(x)|\leq C$ for all $(g,x)\in\Phi\times\mathcal{X}$, and a family of probability measures $\Pi$ over $\mathcal{X}$. Suppose $\{\phi_k\}_{k\in[K]}\subseteq \Phi$ and $\{\mu_k\}_{k\in[K]}\subseteq \Pi$ satisfy that for all $k\in[K]$, $\sum_{t=1}^{k-1}(\Eb_{x\sim\mu_t}[\phi_k(x)])^2\leq \beta$. Then for all $k\in[K]$ and $\omega>0$,
\begin{align*}
    &\sum_{t=1\lor (k-w)}^{k}|\Eb_{x\sim\mu_t}[\phi_t(x)]| \\
    &\leq \mathcal{O}\left(\sqrt{\GDE(\Phi,\Pi,\theta)\beta [k\land (w+1)]}+\min\{w+1,k,\GDE(\Phi,\Pi,\theta)\}C+[k\land (w+1)]\theta\right).
\end{align*}
\end{lemma}

We invoke Lemma~\ref{lemma:GDE} and Lemma~\ref{lemma:concentration} with
\begin{align*}
    &\theta=\sqrt{\frac{1}{w}}, C=H, \\
    &\mathcal{X}=\mathcal{S}\times\mathcal{A},\Phi=(I-\Tc_h)\Fc, \mbox{and } \Pi=\Dc_{\Delta,h}, \\
    &\phi_k=f_h^k-\Tc_h^{k-1}f_{h+1}^k, \mu_k=\mathbf{1}\{\cdot=(x_h^k,a_h^k)\}
\end{align*}
and obtain 
\begin{align*}
    &\sum_{t=1}^{k}\underset{(x_h,a_h)\sim(\pi^t,(*,t-1))}{\Eb}[(f_h^t-\Tc_h^{t-1}f_{h+1}^t)(x_h,a_h)] \\
    &\leq \sum_{t=1}^{k}(f_h^t-\Tc_h^{t-1} f_{h+1}^t)(x_h^t,a_h^t)+\mathcal{O}\left(\sqrt{k\log(k)}\right) \\
    &\leq \mathcal{O}\left(\frac{k}{w}\sqrt{w\cdot\GBE(\Fc,\mathcal{D}_{\Delta,h},\sqrt{1/w})\left(H^2\log[KH|\mathcal{G}|/\delta]+H^2\sup_{t\in[k]}\Delta_{P}^w(t,h)\right)}+\sqrt{w}\right) \\
    &\leq  \mathcal{O}\left(\frac{Hk}{\sqrt{w}}\sqrt{d\log[kH|\mathcal{G}|/\delta]}+\frac{Hk}{\sqrt{w}}\sqrt{d\sup_{t\in[k]}\Delta_{P}^w(t,h)}+\sqrt{w}\right),
\end{align*}
where the second inequality follows from Azuma-Hoeffding inequality, and in the last inequality, we use $\sqrt{a+b}\leq \sqrt{a}+\sqrt{b}$ for any positive $a,b\geq 0$ and we define $d=\GBE(\Fc,\mathcal{D}_{\Delta,h},\sqrt{1/w})$.

Summing over step $h\in[H]$ gives
\begin{align*}
    &\sum_{h=1}^{H}\sum_{k=1}^{k}\underset{(x_h,a_h)\sim(\pi^k,(*,k-1))}{\Eb}[(f_h^k-\Tc_h^{k-1}f_{h+1}^k)(x_h,a_h)] \\
    &\leq \mathcal{O}\left(\frac{H^2k}{\sqrt{w}}\sqrt{d\log[KH|\mathcal{G }|/\delta]}+\frac{H^2k}{\sqrt{w}}\sqrt{d\sup_{t\in[k]}\Delta_{P}^w(t,h)}+H\sqrt{w}\right),
\end{align*}
which completes the proof.

\subsection{Proof of \Cref{coro:thm}}
For ease of exposition, let $d=\GBE(\Fc,\mathcal{D}_{\Delta,h},\sqrt{1/w})$. We adopt average variation $L$ defined in~(\ref{eqn:avg-pathlength}). Then we have 
\begin{align*}
    &\sum_{h=1}^H\sum_{t=1}^{K}\underset{(x_h,a_h)\sim(\pi^t,(*,t-1))}{\Eb}[(f_h^t-\Tc_h^{t-1}f_{h+1}^t)(x_h,a_h)] \\
    &\leq \widetilde{\mathcal{O}}\left(\frac{H^2K}{\sqrt{w}}\sqrt{d}\sqrt{\log|\Gc|}+\frac{H^2K}{\sqrt{w}}\sqrt{dLw^2}+H\sqrt{w}\right) \\
    &\leq \widetilde{\mathcal{O}}\left(H^2K\sqrt{d}\left(\frac{\sqrt{\log|\Gc|}}{\sqrt{w}}+(\sqrt{L}+\frac{1}{HK\sqrt{d}})\sqrt{w}\right)\right).  
\end{align*}
Note first that $\frac{\sqrt{\log|\Gc|}}{\sqrt{L}+\frac{1}{HK\sqrt{d}}}>1$ when $|\Gc|>10$. 

If $\frac{\sqrt{\log|\Gc|}}{\sqrt{L}+\frac{1}{HK\sqrt{d}}}\geq K$, i.e., $\sqrt{L}\leq \frac{1}{K}\left(\sqrt{\log|\Gc|}-\frac{1}{H\sqrt{d}}\right)$, we select $w=K$ and we have
\begin{align*}
&\sum_{h=1}^H\sum_{t=1}^{K}\underset{(x_h,a_h)\sim(\pi^t,(*,t-1))}{\Eb}[(f_h^t-\Tc_h^{t-1}f_{h+1}^t)(x_h,a_h)] \leq \widetilde{O}\left(H^2K^{\frac{1}{2}}d^{\frac{1}{2}}(\log|\Gc|)^{\frac{1}{2}}\right).
\end{align*}

If $\frac{\sqrt{\log|\Gc|}}{\sqrt{L}+\frac{1}{HK\sqrt{d}}}< K$, i.e., $\sqrt{L}> \frac{1}{K}\left(\sqrt{\log|\Gc|}-\frac{1}{H\sqrt{d}}\right)$, we select $w=\lceil\frac{\sqrt{\log|\Gc|}}{\sqrt{L}+\frac{1}{HK\sqrt{d}}}\rceil$ and we have
\begin{align*}
&\sum_{h=1}^H\sum_{t=1}^{K}\underset{(x_h,a_h)\sim(\pi^t,(*,t-1))}{\Eb}[(f_h^t-\Tc_h^{t-1}f_{h+1}^t)(x_h,a_h)] \leq \widetilde{\mathcal{O}}\left(H^2KL^{\frac{1}{4}}d^{\frac{1}{2}}(\log|\Gc|)^{\frac{1}{4}}+H^{\frac{3}{2}}K^{\frac{1}{2}}d^{\frac{1}{4}}(\log|\Gc|)^{\frac{1}{4}}\right).
\end{align*}

\subsection{Proof of Lemma~\ref{lemma:transition_dif}}\label{app:lemma:transition_dif}
\begin{proof}
Fix $(k,h)\in[K]\times[H]$, define reward function $\widetilde{r}_{h'}=r_h^k(x,a)\mathbf{1}\{h'=h\}$ for all $h'\in[H]$. For an episodic MDP $(\mathcal{S},\mathcal{A},H,P^k,\widetilde{r},x_1)$ where $\{P_{h'}^k\}_{h'\in[H]}$ and $\{\widetilde{r}_{h'}\}_{h'\in[H]}$, the state value function and state-action value function of policy $\{\pi_{h'}\}_{h'\in[H]}$ are $\widetilde{V}_{h';,(*,k)}^{\pi}$ and $\widetilde{Q}_{h';,(*,k)}^{\pi}$. Clearly, we have
\begin{align*}
    \left(\underset{(x_h,a_h)\sim(\pi^k,(*,k-1))}{\Eb}-\underset{(x_h,a_h)\sim(\pi^k,(*,k))}{\Eb}\right)[r_h^{k}(x_h,a_h)]= \left(\widetilde{V}_{1;(*,k-1)}^{\pi^k}-\widetilde{V}_{1;(*,k)}^{\pi^k}\right)(x_1).
\end{align*}    

For any function $f : \Sc \times \Ac \rightarrow \Rb$ and any $(k,h,x) \in [K] \times [H] \times \Sc$, define the following operator
\begin{align*}
    (\Jb_{k,h} f)(x)= \langle f(x,\cdot),\pi_h^{k}(\cdot|x)\rangle.
\end{align*}
Note that
\begin{align*}
    &\widetilde{V}_{1;(*,k-1)}^{\pi^k}-\widetilde{V}_{1;(*,k)}^{\pi^k} \\
    &=\Jb_{k,1}\left(\widetilde{Q}_{1;(*,k-1)}^{\pi^k}-\widetilde{Q}_{1;(*,k)}^{\pi^k}\right) \\
    &=\Jb_{k,1}\left(\Pb_{1}^{k-1}\widetilde{V}_{2;(*,k-1)}^{\pi^k}-\Pb_{1}^{k}\widetilde{V}_{2;(*,k)}^{\pi^k}\right)  \\
    &=\Jb_{k,1}\Pb_{1}^{k-1}\left(\widetilde{V}_{2;(*,k-1)}^{\pi^k}-\widetilde{V}_{2;(*,k)}^{\pi^k}\right)+\Jb_{k,1}\left(\Pb_{1}^{k-1}-\Pb_{1}^{k}\right)\widetilde{V}_{2;(*,k)}^{\pi^k} \\
    &=\prod_{i=1}^{h}\left(\Jb_{k,i}\Pb_i^{k-1}\right)\underbrace{\left(\widetilde{V}_{h+1;(*,k-1)}^{\pi^k}-\widetilde{V}_{h+1;(*,k)}^{\pi^k}\right)}_{=0}+\sum_{i=1}^{h}\prod_{\ell=1}^{i-1}\left(\Jb_{k,\ell}\Pb_{\ell}^{k-1}\right)\Jb_{k,i}\left(\Pb_i^{k-1}-\Pb_i^k\right)\widetilde{V}_{i+1,(*,k)}^{\pi^k} \\
    &=\sum_{i=1}^{h-1}\prod_{\ell=1}^{i-1}\left(\Jb_{k,\ell}\Pb_{\ell}^{k-1}\right)\Jb_{k,i}\left(\Pb_i^{k-1}-\Pb_i^k\right)\widetilde{V}_{i+1,(*,k)}^{\pi^k}. 
\end{align*}
where in the second equality we use the fact that reward $\widetilde{r}$ is identical. I.e.,
\begin{align*}
    &\left(\widetilde{V}_{1;(*,k-1)}^{\pi^k}-\widetilde{V}_{1;(*,k)}^{\pi^k}\right)(x_1) \\
    &=\sum_{i=1}^{h-1}\underset{(x_i,a_i)\sim (\pi^k,(*,k-1))}{\Eb}\left[\left(\left(\Pb_i^{k-1}-\Pb_i^k\right)\widetilde{V}_{i+1,(*,k)}^{\pi^k}\right)(x_i,a_i)\right] \\
    &\leq \sum_{i=1}^{h-1} \sup_{x,a}\norm{P_i^{k-1}(\cdot|x,a)-P_i^{k}(\cdot|x,a)}_1 \\
\end{align*}
Therefore,
\begin{align*}
    &\sum_{k=1}^{k'}\sum_{h=1}^{H}\left(\underset{(x_h,a_h)\sim(\pi^k,(*,k-1))}{\Eb}-\underset{(x_h,a_h)\sim(\pi^k,(*,k))}{\Eb}\right)[r_h^{k}(x_h,a_h)] \\
    &\leq \sum_{k=1}^{k'}\sum_{h=1}^{H}\sum_{i=1}^{h-1}\sup_{x,a}\norm{P_i^{k-1}(\cdot|x,a)-P_i^{k}(\cdot|x,a)}_1 \\
    &\leq \sum_{k=1}^{k'}\sum_{h=1}^{H}\sum_{i=1}^{H}\sup_{x,a}\norm{P_i^{k-1}(\cdot|x,a)-P_i^{k}(\cdot|x,a)}_1 \\ 
    &\leq \sum_{h=1}^{H} \left(\sum_{k=1}^{k'}\sum_{i=1}^{H}\sup_{x,a}\norm{P_i^{k-1}(\cdot|x,a)-P_i^{k}(\cdot|x,a)}_1\right) \\
    &\leq H\Delta_P(k').
\end{align*}
\end{proof}

\subsection{Proofs of concentration lemmas}\label{app:concentration}
The Freedman's inquaulity controls the sum of martingale difference by the sum of their variance.
\begin{lemma}[Freedman's inequality (\cite{Chi_Jin:bellman_eluder:2020})]
    Let $\{Z_t\}_{t\in[T]}$ be a real-valued martingale difference sequence adapted to filtration $\Fc_t$, and let $\Eb_t[\cdot]=\Eb[\cdot|\Fc_t]$. If $|Z_t|\leq R$ almost surely, then for any $\eta\in (0,R)$, it holds that with probability at least $1-\delta$,
    \begin{align*}
        \sum_{t=1}^T Z_t\leq \Oc\left(\eta\sum_{t=1}^{T}\Eb_{t-1}[Z_t^2]+\frac{\log(\delta^{-1})}{\eta}\right).
    \end{align*}
\end{lemma}

\subsubsection{Proof of Lemma~\ref{lemma:Qstar}}
\begin{proof}
Define
\begin{align*}
    \#_{k,h}(x_h^t,a_h^t):= r_h^{k}(s_h^t,a_h^t)+\underset{x'\sim P_h^t(\cdot|x_h^t,a_h^t)}{\Eb}\max_{a'\in\mathcal{A}}Q_{h+1;(*,k)}(x',a').
\end{align*}
Fix a tuple $(k,h,g)\in[K]\times[H]\times\Gc$. Let
\begin{align*}
W_t(h,f):&=\left[g_h(x_h^t,a_h^t)-r_h^{k}-\max_{a'\in\mathcal{A}}Q_{h+1;(*,k)}(x_{h+1}^t,a'))\right]^2-\left[\#_{k,h}(x_h^t,a_h^t)-r_h^{k}-\max_{a'\in\mathcal{A}}Q_{h+1;(*,k)}(x_{h+1}^t,a'))\right]^2 \\
&=[g_h(x_h^t,a_h^t)-\#_{k,h}(x_h^t,a_h^t)]\left[g_h(x_h^t,a_h^t)+\#_{k,h}(x_h^t,a_h^t)-2\left(r_h^k+\max_{a'\in\mathcal{A}}Q_{h+1;\stark}(x_{h+1}^t,a')\right)\right]
\end{align*}
and $\Fc_{t,h}$ be the filtration induced by $\{x_1^i,a_1^i,\cdots,x_H^i\}_{i\in[t-1]}\cup\{x_1^t,a_1^t,\cdots,x_h^t,a_h^t\}\cup\{r_h^i\}_{h\in[H]}^{i\in[t-1]}$. We have
\begin{align*}
&\Eb[W_t(h,g)|\Fc_{t,h}]=\left[\left({g}_h-\#_{k,h}\right)(x_h^t,a_h^t)\right]^2, \\
&\Var[W_t(h,g)|\Fc_{t,h}]\leq 36H^2\Eb[W_t(h,g)|\Fc_{t,h}].
\end{align*}
By Freedman's inequality, with probability at least $1-\delta$,
\begin{align*}
    &\left|\sum_{t=1\lor (k-w)}^{k}W_t(h,g)-\sum_{t=1\lor (k-w)}^{k}\left[\left({g}_h(x_h^t,a_h^t)-\#_{k,h}\right)(x_h^t,a_h^t)\right]^2\right| \\
    &\qquad \leq \Oc\left(H\sqrt{\log(1/\delta)\sum_{t=1\lor (k-w)}^{k}\left[\left({g}_h(x_h^t,a_h^t)-\#_{k,h}\right)(x_h^t,a_h^t)\right]^2  }+\log(1/\delta)\right).
\end{align*}
Taking union bound over $[K]\times[H]\times\Gc$,
\begin{align*}
    &\left|\sum_{t=1\lor (k-w)}^{k}W_t(h,g)-\sum_{t=1\lor (k-w)}^{k}\left[\left({g}_h(x_h^t,a_h^t)-\#_{k,h}\right)(x_h^t,a_h^t)\right]^2\right| \\
    &\qquad \leq \Oc\left(H\sqrt{\iota\sum_{t=1\lor (k-w)}^{k}\left[\left({g}_h(x_h^t,a_h^t)-\#_{k,h}\right)(x_h^t,a_h^t)\right]^2  }+\iota\right),
\end{align*}
where $\iota=\log(HK|\Gc|/\delta)$. We have
\begin{align*}
&-\sum_{t=1\lor (k-w)}^{k}W_t(h,g)  \\
&\leq -\sum_{t=1\lor (k-w)}^{k}\left[\left({g}_h(x_h^t,a_h^t)-\#_{k,h}\right)(x_h^t,a_h^t)\right]^2  +\Oc\left(H\sqrt{\iota\sum_{t=1\lor (k-w)}^{k}\left[\left({g}_h(x_h^t,a_h^t)-\#_{k,h}\right)(x_h^t,a_h^t)\right]^2  }+\iota\right) \\
&\leq \Oc(H^2\iota). 
\end{align*}
I.e., 
\begin{align*}
&\sum_{t=1\lor (k-w)}^{k}\left[\#_{k,h}(x_h^t,a_h^t)-r_h^k-\max_{a'\in\mathcal{A}}Q_{h+1;\stark}(x_{h+1}^t,a')\right]^2 \\
&\quad \qquad \leq \sum_{t=1\lor (k-w)}^{k}\left[{g}_h(x_h^t,a_h^t)-r_h^k-\max_{a'\in\mathcal{A}}Q_{h+1;\stark}(x_{h+1}^t,a')\right]^2 + \Oc(H^2\iota).
\end{align*}
Therefore,
\begin{align*}
    &\sum_{t=1\lor (k-w)}^{k}\left[Q_{h;(*,k)}(x_h^t,a_h^t)-r_h^k-\max_{a'\in\mathcal{A}}Q_{h+1;\stark}(x_{h+1}^t,a')\right]^2 \\
    &\leq \sum_{t=1\lor (k-w)}^{k}\left[\#_{k,h}(x_h^t,a_h^t)-r_h^k-\max_{a'\in\mathcal{A}}Q_{h+1;\stark}(x_{h+1}^t,a')\right]^2 +2H^2\Delta_P^w(k,h) \\
    &\leq \sum_{t=1\lor (k-w)}^{k}\left[{g}_h(x_h^t,a_h^t)-r_h^k-\max_{a'\in\mathcal{A}}Q_{h+1;\stark}(x_{h+1}^t,a')\right]^2 + 2H^2\Delta_P^w(k,h)+\Oc(H^2\iota),
\end{align*}
where the first inequality follows from Lemma~\ref{lemma:square-dif} and Eqn.~(\ref{eqn:pathlength-3}). By the definition of $\Bc^{k}$ and $\beta=cH^2\log\frac{KH|\Gc|}{\delta}$ with some large absolute constant $c$, we conclude that with probability at least $1-\delta$, $Q_{(*,k)}\in\Bc^{k}$ for all $k\in[K]$.
\end{proof}

\subsubsection{Proof of Lemma~\ref{lemma:concentration}}
\begin{proof}
Define
\begin{align*}
    \#_{k,h}^f(x_h^t,a_h^t)= r_h^{k}(s_h^t,a_h^t)+\underset{x'\sim P_h^t(x_h^t,a_h^t)}{\Eb}\max_{a'\in\mathcal{A}}f_{h+1}(s',a').
\end{align*}
Fix a tuple $(k,h,f)\in[K]\times[H]\times\Gc$. Let
\begin{align*}
W_t(h,f):&=\left[f_h(x_h^t,a_h^t)-r_h^{k}-\max_{a'\in\mathcal{A}}f_{h+1}(x_{h+1}^t,a')\right]^2-\left[\#_{k,h}^f(x_h^t,a_h^t)-r_h^{k}-\max_{a'\in\mathcal{A}}f_{h+1}(x_{h+1}^t,a')\right]^2 \\
&=[f_h(x_h^t,a_h^t)-\#_{k,h}^f(x_h^t,a_h^t)]\left[f_h(x_h^t,a_h^t)+\#_{k,h}^f(x_h^t,a_h^t)-2\left(r_h^k+\max_{a'\in\mathcal{A}}f_{h+1}(x_{h+1}^t,a')\right)\right]
\end{align*}
and $\Fc_{t,h}$ be the filtration induced by $\{x_1^i,a_1^i,\cdots,x_H^i\}_{i\in[t-1]}\cup\{x_1^t,a_1^t,\cdots,x_h^t,a_h^t\}\cup\{r_h^i\}_{h\in[H]}^{i\in[t-1]}$. We have
\begin{align*}
&\Eb[W_t(h,f)|\Fc_{t,h}]=\left[\left({f}_h-\#_{k,h}^f\right)(x_h^t,a_h^t)\right]^2, \\
&\Var[W_t(h,f)|\Fc_{t,h}]\leq 36H^2\Eb[W_t(h,g)|\Fc_{t,h}].
\end{align*}
By Freedman's inequality, we have
\begin{align*}
    &\left|\sum_{t=1\lor (k-w)}^{k}W_t(h,f)-\sum_{t=1\lor (k-w)}^{k}\left[({f}_h-\#_{k,h}^f)(x_h^t,a_h^t)\right]^2\right|  \\
    &\qquad\leq \Oc\left(H\sqrt{\log(1/\delta)\sum_{t=1\lor (k-w)}^{k}\left[({f}_h-\#_{k,h}^f)(x_h^t,a_h^t)\right]^2  }+\log(1/\delta)\right).
\end{align*}
Taking union bound over $[K]\times[H]\times\Gc$, we have
\begin{align*}
    &\left|\sum_{t=1\lor (k-w)}^{k}W_t(h,g)-\sum_{t=1}^{k}\left[({f}_h-\#_{k,h}^f)(x_h^t,a_h^t)\right]^2\right| \leq \Oc\left(H\sqrt{\iota\sum_{t=1\lor (k-w)}^{k}\left[({f}_h-\#_{k,h}^f)(x_h^t,a_h^t)\right]^2  }+\iota\right),
\end{align*}
where $\iota=\log(KH|\Gc|/\delta)$.

Note that
\begin{align*}
    &\sum_{t=1\lor (k-w-1)}^{k-1}W_t(h,f^k) \\
    &=\sum_{t=1\lor (k-w-1)}^{k-1}\left[f_h^k(x_h^t,a_h^t)-r_h^{k-1}-\max_{a'\in\mathcal{A}}f_{h+1}^k(x_{h+1}^t,a')\right]^2 \\
    &\qquad \qquad \qquad \qquad -\sum_{t=1\lor (k-w-1)}^{k-1}\left[\#_{k-1,h}^{f^k}(x_h^t,a_h^t)-r_h^{k-1}-\max_{a'\in\mathcal{A}}f_{h+1}^k(x_{h+1}^t,a')\right]^2 \\
    &\leq\sum_{t=1\lor (k-w-1)}^{k-1}\left[f_h^k(x_h^t,a_h^t)-r_h^{k-1}-\max_{a'\in\mathcal{A}}f_{h+1}^k(x_{h+1}^t,a')\right]^2 \\
    &\qquad \qquad \qquad \qquad-\sum_{t=1\lor (k-w-1)}^{k-1}\left[\Tc_{h}^{k-1}f_{h+1}^k(x_h^t,a_h^t)-r_h^{k-1}-\max_{a'\in\mathcal{A}}f_{h+1}^k(x_{h+1}^t,a')\right]^2  +2H^2\Delta_P^w(k-1,h) \\
    &\leq \sum_{t=1\lor (k-w-1)}^{k-1}\left[f_h^k(x_h^t,a_h^t)-r_h^{k-1}-\max_{a'\in\mathcal{A}}f_{h+1}^k(x_{h+1}^t,a')\right]^2 \\
    &\qquad \qquad \qquad \qquad-\inf_{g\in\Gc}\sum_{t=1\lor (k-w-1)}^{k-1}\left[g_h(x_h^t,a_h^t)-r_h^{k-1}-\max_{a'\in\mathcal{A}}f_{h+1}^k(x_{h+1}^t,a')\right]^2 +2H^2\Delta_P^w(k-1,h)  \\
    &\leq \beta+4H^2\Delta_P^w(k-1,h),
\end{align*}
where the first inequality follows from Lemma~\ref{lemma:square-dif} and Eqn.~(\ref{eqn:pathlength-3}), the second inequality follows from Assumption~\ref{assump:completeness-2}, and the last inequality follows from the definition of $\Bc^{k-1}$.

Therefore, 
\begin{align*}
   \sum_{t=1\lor (k-w-1)}^{k-1}\left[(f_h^k-\#_{k-1,h}^{f^k})(x_h^t,a_h^t)\right]^2\leq \beta+4H^2\Delta_P^w(k-1,h)+\Oc\left(H^2\iota\right).
\end{align*}
Finally, we use Lemma~\ref{lemma:square-dif} once more and obtain
\begin{align*}
    &\sum_{t=1\lor (k-w-1)}^{k-1}\left[(f_{h}^k-\Tc_h^{k-1} f_{h+1}^k)(x_h^t,a_h^t)\right]^2 \\
    &\leq \sum_{t=1\lor (k-w-1)}^{k-1}\left[(f_{h}^k-\#_{k-1,h}^{f^k})(x_h^t,a_h^t)\right]^2+2H^2\Delta_P^w(k-1,h) \\
    &\leq 6H^2\Delta_P^w(k-1,h)+\Oc(\beta).
\end{align*}
\end{proof}

\subsection{Proof of Lemma~\ref{lemma:GDE}}\label{app:GDE}
The proof in the subsection essentially follows the same arguments as in \citep{Chi_Jin:bellman_eluder:2020}, and we adapt it to the sliding window scenario.

\begin{lemma}\label{lemma:GDE-1}
Given a function class $\Phi$ defined on $\mathcal{X}\times \mathcal{Y}$, and a family of probability measures $\Pi$ over $\mathcal{X}$. Suppose sequence $\{\phi_k\}_{k\in[K]}\subseteq \Phi$ and $\{\mu_k\}_{k\in[K]}\subseteq \Pi$ satisfy that for all $k\in [K]$, $\sum_{t=1\lor (k-w-1)}^{k-1}(\Eb_{x\sim\mu_t}[\phi_k(x)])^2\leq \beta$. Then for all $k\in[K]$,
\begin{align*}
    \sum_{t=1\lor (k-w)}^{k}\mathbf{1}\{|\Eb_{x\sim\mu_t}[\phi_t(x)]|>\epsilon\}\leq (\frac{\beta}{\epsilon^2}+1)\GDE(\Phi,\Pi,\epsilon)
\end{align*}
\end{lemma}
\begin{proof}
First, suppose for all $k\in[\kappa]$, $\sum_{t=1}^{k-1}(\Eb_{x\sim\mu_t}[\phi_k(x)])^2\leq \beta$, we show that if for some $k\in[\kappa]$ we have $|\Eb_{x\sim\mu_k}[\phi_k(x)]|>\epsilon$, then $\mu_k$ is $\epsilon$-dependent on at most $\lceil\beta/\epsilon^2\rceil-1$ disjoint subsequences in $\{\mu_1,\ldots,\mu_{k-1}\}$. By definition of GDE, if $|\Eb_{x\sim\mu_k}[\phi_k(x)]|>\epsilon$ and $\mu_k$ is $\epsilon$-dependent on a subsequence $\{\nu_1,\ldots,\nu_{\ell}\}$ of $\{\mu_{1},\ldots, \mu_{k-1}\}$, then we should have $\sum_{t=1}^{\ell}(\Eb_{x\sim\nu_t}[\phi_k(x)])^2> \epsilon^2$. It implies that if $\mu_k$ is $\epsilon$-dependent on $L$ disjoint subsequences in $\{\mu_{1},\ldots, \mu_{k-1}\}$, we have
\begin{align*}
    \beta\geq \sum_{t=1\lor (k-w-1)}^{k-1}(\Eb_{x\sim\mu_t}[\phi_k(x)])^2> L\epsilon^2,
\end{align*}
which implies $L\leq\lceil\beta/\epsilon^2\rceil-1$.

Second, we show that for any sequence $\{\nu_1,\ldots,\nu_\kappa\}\subseteq \Pi$, there exists $j\in[\kappa]$ such that $\nu_j$ is $\epsilon$-dependent on at least $L=\lceil (\kappa-1)/\GDE(\Phi,\Pi,\epsilon)\rceil$ disjoint subsequences in $\{\nu_{1},\ldots,\nu_{j-1}\}$. We prove the argument by the following artificial procedure: we start with singleton sequences $B_1=\{\nu_1\}$, $B_2=\{\nu_2\}$, $\ldots$, $B_L=\{\nu_L\}$ and $j=L+1$. For each $j$, if $\nu_j$ is $\epsilon$-dependent on $B_1,B_2,\ldots,B_L$, we already achieved the goal and we stop; otherwise, we pick an $i\in[L]$ such that $\nu_j$ is $\epsilon$-independent of $B_i$ and update $B_i\cup\{\nu_j\}$. Then we increment $j$ by 1 continue this process. By the definition of GDE dimension, the size of each $B_1,B_2,\ldots,B_L$ cannot get bigger than $\GDE(\Phi,\Pi,\epsilon)$ at any point in this process. Therefore, the process stops before or on $j=L\GDE(\Phi,\Pi,\epsilon)+1\leq \kappa$.

Now fix $k\in[K]$ and let $\{\nu_1,\ldots,\nu_\kappa\}$ be the subsequence of $\{\mu_{1\lor (k-w)},\ldots,\mu_k\}$, consisting of elements for which $|\Eb_{x\sim\mu_t}[|\phi_t(x)]|>\epsilon$ and the corresponding bijective function is $\theta:[\kappa]\mapsto [1\lor (k-w):k]$. Note that for all $\ell\in[\kappa]$, we have $|\Eb_{x\sim\nu_\ell}[|\phi_{\theta(\ell)}(x)]|>\epsilon$ and 
\begin{align*}
    \sum_{t=1}^{\ell-1}(\Eb_{x\sim \nu_t}[\phi_{\theta(\ell)}(x)])\leq \sum_{t=1\lor \theta(\ell)-w-1}^{\theta(\ell)-1}(\Eb_{x\sim \mu_t}[\phi_{\theta(\ell)}(x)])\leq \beta.
\end{align*}
Using the first claim, we know that each $\nu_j$ is $\epsilon$-dependent on at most $L<\lceil\beta/\epsilon^2\rceil-1$ disjoint subsequences of $\{\nu_1,\nu_2,\ldots,\nu_{j-1}\}$. Using the second claim, we know that there exists $j\in[\kappa]$ such that $\nu_j$ is $\epsilon$-dependent on at least $\lceil (\kappa-1)/\GDE(\Phi,\Pi,\epsilon)\rceil$ disjoint subsequences of $\{\nu_1,\nu_2,\ldots,\nu_{j-1}\}$. Therefore, we have $\lceil (\kappa-1)/\GDE(\Phi,\Pi,\epsilon)\rceil\leq \lceil\beta/\epsilon^2\rceil-1$, which implies
\begin{align*}
    \kappa<(\frac{\beta}{\epsilon^2}+1)\GDE(\Phi,\Pi,\epsilon).
\end{align*}
\end{proof}

\noindent \emph{Proof of Lemma~\ref{lemma:GDE}.} Fix $k\in[K]$ and let $d=\GDE(\Phi,\Pi,\epsilon)$. Sort the sequence
\begin{align*}
    \left\{|\Eb_{x\sim\mu_{1\lor (k-w)}}[\phi_{1\lor (k-w)}(x)]|, \ldots,|\Eb_{x\sim\mu_k}[\phi_k(x)]|\right\}
\end{align*}
in decreasing order and denote it by $\{e_{1},e_2,\ldots,e_{k\land (w+1)}\}$ ($e_{1}\geq e_2\geq \cdots\geq e_{k\land (w+1)}$). Note that
\begin{align*}
    \sum_{t=1\lor (k-w)}^{k}|\Eb_{x\sim\mu_t}[\phi_t(x,y)]|&=\sum_{t=1}^{k\land (w+1)}e_t=\sum_{t=1}^{k\land (w+1)}e_t\mathbf{1}\{e_t\leq\theta\}+\sum_{t=1}^{k\land (w+1)}e_t\mathbf{1}\{e_t>\theta\} \\
    &\leq [k\land (w+1)]\theta+\sum_{t=1}^{k\land (w+1)}e_t\mathbf{1}\{e_t>\theta\} 
\end{align*}
For $t\in[k]$, we show that if $e_t>\theta$, then we have $e_t\leq\min\{\sqrt{\frac{d\beta}{t-d}},C\}$. Assume $t\in[k]$ satisfies $e_t>\theta$. Then there exists an $\alpha$ such that $e_t>\alpha\geq \theta$. By Lemma~\ref{lemma:GDE-1}, we have
\begin{align*}
    t\leq \sum_{i=1}^{k\land (w+1)}\mathbf{1}\{e_i>\alpha\}\leq (\frac{\beta}{\alpha^2}+1)\GDE(\Phi,\Pi,\alpha)\leq (\frac{\beta}{\alpha^2}+1)\GDE(\Phi,\Pi,\omega),
\end{align*}
which implies $\alpha\leq \sqrt{\frac{d\beta}{t-d}}$. Letting $\alpha\rightarrow e_t$, we have $e_t\leq \sqrt{\frac{d\beta}{t-d}}$. Besides, recall $e_t\leq C$, so we have $e_t\leq \min\{\sqrt{\frac{d\beta}{t-d}},C\}$.

Finally, we have
\begin{align*}
    \sum_{t=1}^{k\land (w+1)}e_t\mathbf{1}\{e_t>\omega\}&\leq \min\{d,k,w+1\}C+\sum_{t=d+1}^{k\land (w+1)}\sqrt{\frac{d\beta}{t-d}} \\
    &\leq \min\{d,k,w+1\}C+\sqrt{d\beta}\int_{0}^{k\land (w+1)}\frac{1}{\sqrt{t}}\dif t \\
    &=\min\{d,k,w+1\}C+2\sqrt{d\beta [k\land (w+1)]}.
\end{align*}

\subsection{Auxiliary Lemmas}\label{app:aux-lemmas}
\begin{lemma}\label{lemma:square-dif}
Suppose $P$ and $Q$ are two probability distributions of a random variable $x$, then 
\begin{align*}
    \left|\left(\underset{x\sim P}{\Eb}f(x)-C\right)^2-\left(\underset{x\sim Q}{\Eb}f(x)-C\right)^2\right|\leq (2f_m+2|C|)f_m\cdot\TV(P,Q), 
\end{align*}
where $f_m=\sup_x |f(x)|$.
\end{lemma}
\begin{proof}
Note that
\begin{align*}
    &\left|\left(\underset{x\sim P}{\Eb}f(x)-C\right)^2-\left(\underset{x\sim Q}{\Eb}f(x)-C\right)^2\right| \\
    &= \left|\left(\underset{x\sim P}{\Eb}f(x)+\underset{x\sim Q}{\Eb}f(x)-2C\right)\left(\underset{x\sim P}{\Eb}f(x)-\underset{x\sim Q}{\Eb}f(x)\right)\right| \\
    &\leq (2f_m+2|C|)\left|\int_x f(x)(\dif P-\dif Q)\right| \\
    &\leq (2f_m+2|C|)f_m\cdot\TV(P,Q).
\end{align*}
\end{proof}

\begin{lemma}[Generalized policy loss decomposition]\label{lemma:dr-decom-2}
For any $t,k$, we have
\begin{align*}
    f_{1}^{t}(x_1,\pi_1^t(x_1))-V_{1;\stark}^{\pi^t}(x_1)&=\sum_{h=1}^{H}\underset{(x_h,a_h)\sim(\pi^{t},(*,k))}{\Eb}\left[(f_{h}^{t}-r_{h;\stark}-\Pb_h^k f_{h+1}^{t})(x_h,a_h)\right],
\end{align*}
where $\pi_t:=\pi_{f^t}$, the greedy policy under function approximation $f^t$.
\end{lemma}
\begin{proof}
Note that
\begin{align*}
    &\sum_{h=1}^{H}\underset{(x_h,a_h)\sim(\pi^{t},(*,k))}{\Eb}\left[(f_{h}^{t}-r_{h;\stark}-\Pb_h^k f_{h+1}^{t})(x_h,a_h)\right] \\
    &=\sum_{h=1}^{H}\underset{(x_h,a_h,x_{h+1})\sim(\pi^{t},(*,k))}{\Eb}\left[f_{h}^{t}(x_h,a_h)-r_{h;\stark}(x_h,a_h)-\max_{a\in\mathcal{A}}f_{h+1}^{t}(x_{h+1},a)\right] \\
    &=\sum_{h=1}^{H}\underset{(x_h,a_h)\sim(\pi^{t},(*,k))}{\Eb}\left[f_{h}^{t}(x_h,a_h)-r_{h;\stark}(x_h,a_h)-\underset{(x_{h+1},a_{h+1})\sim(\pi^{t},(*,k))}{\Eb}[f_{h+1}^{t}(x_{h+1},a_{h+1})]\right] \\
    &=\underset{(x_{1:H},a_{1:H})\sim(\pi^{t},(*,k))}{\Eb}\left[\sum_{h=1}^{H}\left(f_{h}^{t}(x_h,a_h)-f_{h+1}^{t}(x_{h+1},a_{h+1})\right)\right] -\underset{(x_{1:H},a_{1:H})\sim(\pi^{t},(*,k))}{\Eb}\left[\sum_{h=1}^{H}r_{h;\stark}(x_h,a_h)\right] \\
    &=f_{1}^t(x_1,\pi_1^t(x_1))-V_{1;\stark}^{\pi^t}(x_1),
\end{align*}
where the second equality follows from $\pi_t=\pi_{f^t}$.
\end{proof}

\section{Bandit Feedback}\label{app:bandit}
We extend our algorithm to bandit feedback scenario, and the pseudocode is presented in Algorithm~\ref{Alg:2}. In bandit feedback scenario, the reward function $r_h^k(\cdot,\cdot)$ is no long available, and the agent can only get access to the reward obtained from the trajectory. Therefore, the non-stationarity of rewards plays an important role in the construction of the sliding window Bellman error and the confidence set. Specifically, we replace the sliding window squared Bellman error (\ref{eqn:loss}) with 
\begin{align}
&\Lc_{\Dc_h}(\xi_h,\zeta_{h+1})= \sum_{t=1\lor (k-w)}^{k}\left(\xi_h(x_h^{t},a_h^{t})-r_h^t -\max_{a'\in\mathcal{A}}\zeta_{h+1}(x_{h+1}^{t},a')\right)^2, \nonumber 
\end{align}
where $r_h^t$ is the reward obtained at step $h$ in episode $t$. Moreover, the local regression constraint is
\begin{align*}
&\Lc_{\Dc_h}(f_h,f_{h+1})\leq \inf_{g\in\Gc_h} \Lc_{\Dc_h}(g,f_{h+1})+\beta    +2H^2\Delta_P^w(k,h)+2H\Delta_R^w(k,h),
\end{align*}
where $\beta$ is a confidence parameter, $\Delta_P^w$ is the local variation budget defined in (\ref{eqn:pathlength-3}) and $\Delta_R^w$ is defined as
\begin{align*}
 \Delta_{P}^{w}(k,h) 
    =\sum_{t=1\lor (k-w)}^{k}\sup_{x\in\mathcal{S},a\in\mathcal{A}}|(r_h^{k}-r_h^{t})(x,a)| .
\end{align*}
\subsection{Algorithm and Theorem}
\begin{algorithm}
\begin{algorithmic}[1]
\caption{\alg (bandit feedback)}\label{Alg:2}
\STATE {{\bf Input:} $\Dc_1,\cdots,\Dc_H\leftarrow \varnothing$, $\Bc^0\leftarrow \Fc$.} 
\FOR{{\bf episode} $k$ from 1 to $K$}
\STATE{{\bf Choose} $\pi^k=\pi_{f^k}$,\newline where $f^k=\arg\max_{f\in\Bc^{k-1}}f_1(x_1,\pi_f(x_1))$.} 
\STATE{{\bf Collect} a trajectory $(x_1^{k},a_1^{k},\cdots,x_H^{k},a_H^{k},x_{H+1}^{k})$ by following $\pi^k$ and reward function $\{r_h^k\}_{h\in[H]}$.} 
\STATE{{\bf Augment} $\Dc_h=\Dc_h\cup\{(x_h^{k},a_h^{k},x_{h+1}^{k})\}$, $\forall h\in[H]$.}
\STATE{{\bf} Update $\Bc^k =\{f\in\mathcal{F}:\Lc_{\Dc_h}(f_h,f_{h+1})\leq \inf_{g\in\Gc_h} \Lc_{\Dc_h}(g,f_{h+1})+\beta+2H^2\Delta_P^w(k,h)+2H\Delta_R^w(k,h)$, $\forall h\in[H]\}$,
\newline
where $\Lc_{\Dc_h}(\xi_h,\zeta_{h+1})= \sum_{t=1\lor (k-w)}^{k}\left(\xi_h(x_h^{t},a_h^{t})-r_h^t -\max_{a'\in\mathcal{A}}\zeta_{h+1}(x_{h+1}^{t},a')\right)^2$ }
\ENDFOR
\end{algorithmic}
\end{algorithm}

\begin{theorem}\label{app-thm:DR-2}
Under Assumption~\ref{assump:realizability} and Assumption~\ref{assump:completeness-2}, there exists an absolute constant $c$ such that for any $\delta\in(0,1]$, $K\in\mathbb{N}$, if we choose $\beta=cH^2\log\frac{KH|\Gc|}{\delta}$ in \algg, then with probability at least $1-\delta$, for all $k\in[K]$, when $k\geq \min\{w+1,\GBE(\Fc,\mathcal{D}_{\Delta,h},\sqrt{1/w})\}$ we have
\begin{align*}
  \DR(k) =&\Delta_R(k)+H\Delta_P(k) \\
 &+\mathcal{O}\left(H\sqrt{w}  +\frac{H^2k}{\sqrt{w}}\sqrt{d\log[KH|\mathcal{G}|/\delta]} 
    +\frac{H^2k}{\sqrt{w}}\sqrt{d\sup_{t\in[k]}\Delta_{P}^w(t,h)}+\frac{H^{3/2}k}{\sqrt{w}}\sqrt{d\sup_{t\in[k]}\Delta_{R}^w(t,h)}\right).
\end{align*}
where $d=\GBE(\Fc,\mathcal{D}_{\Delta,h},\sqrt{1/w})$.
\end{theorem}

\subsection{Proof of Theorem~\ref{thm:DR-2}}
Following the same argument in Appendix~\ref{app:thm-DR} gives
\begin{align*}
    &\DR(k)=H+\Delta_R(k)+H\Delta_P(k)+\underbrace{\sum_{t=1}^{k}\left(V_{1;(*,t-1)}^{\pi^{(*,t-1)}}-V_{1;(*,t-1)}^{\pi^t}\right)(x_1)}_{\URM{1}}.
\end{align*}
In the sequel, we strive to bound term $\URM{1}$. We first introduce a different probability distribution shift lemma. Compared to \Cref{lemma:change-transition}, the new lemma is more general and can handle the bandit feedback scenario.
\begin{lemma}\label{lemma:square-dif-2}
Suppose $P$ and $Q$ are two probability distributions of a random variable $x$, then 
\begin{align*}
    \left|\left(\underset{x\sim P}{\Eb}f(x)+\Eb g_1(y) -C\right)^2-\left(\underset{x\sim Q}{\Eb}f(x)+\Eb g_2(y) -C\right)^2\right|\leq (2f_m+2g_m+2|C|)f_m\cdot\TV(P,Q), 
\end{align*}
where $f_m=\sup_x |f(x)|$, $g_m=\max_{i=1,2}\sup_y g_i(y)$.
\end{lemma}
\begin{proof}
Note that
\begin{align*}
    &\left|\left(\underset{x\sim P}{\Eb}f(x)-C\right)^2-\left(\underset{x\sim Q}{\Eb}f(x)-C\right)^2\right| \\
    &= \left|\left(\underset{x\sim P}{\Eb}f(x)+\underset{x\sim Q}{\Eb}f(x)+\Eb g_1(y)+\Eb g_2(y)-2C\right)\left(\underset{x\sim P}{\Eb}f(x)-\underset{x\sim Q}{\Eb}f(x)+\Eb g_1(y)-\Eb g_2(y)\right)\right| \\
    &\leq (2f_m+2g_m+2|C|)\left(\left|\int_x f(x)(\dif P-\dif Q)\right|+\sup_y |g_1(y)-g_2(y)|\right) \\
    &\leq (2f_m+2g_m+2|C|)(f_m\cdot\TV(P,Q)+\sup_y |g_1(y)-g_2(y)|).
\end{align*}
\end{proof}

Thanks to \Cref{lemma:square-dif-2}, we are able to obtain the following two lemmas.
\begin{lemma}\label{lemma:Qstar-2}
If $\beta=cH^2\log\frac{KH|\Gc|}{\delta}$, then with probability at least $1-\delta$, we have $Q_{\stark}^*\in\Bc^{k}$ for all $k\in[K]$.
\end{lemma}
\begin{proof}
Define
\begin{align*}
    \#_{k,h}(x_h^t,a_h^t):= \Eb[r_h^{t}(s_h^t,a_h^t)]+\underset{x'\sim P_h^t(\cdot|x_h^t,a_h^t)}{\Eb}\max_{a'\in\mathcal{A}}Q_{h+1;(*,k)}(x',a').
\end{align*}
Fix a tuple $(k,h,g)\in[K]\times[H]\times\Gc$. Let
\begin{align*}
W_t(h,f):&=\left[g_h(x_h^t,a_h^t)-r_h^{t}-\max_{a'\in\mathcal{A}}Q_{h+1;(*,k)}(x_{h+1}^t,a'))\right]^2-\left[\#_{k,h}(x_h^t,a_h^t)-r_h^{t}-\max_{a'\in\mathcal{A}}Q_{h+1;(*,k)}(x_{h+1}^t,a'))\right]^2 \\
&=[g_h(x_h^t,a_h^t)-\#_{k,h}(x_h^t,a_h^t)]\left[g_h(x_h^t,a_h^t)+\#_{k,h}(x_h^t,a_h^t)-2\left(r_h^t+\max_{a'\in\mathcal{A}}Q_{h+1;\stark}(x_{h+1}^t,a')\right)\right]
\end{align*}
and $\Fc_{t,h}$ be the filtration induced by $\{x_1^i,a_1^i,\cdots,x_H^i\}_{i\in[t-1]}\cup\{x_1^t,a_1^t,\cdots,x_h^t,a_h^t\}\cup\{r_h^i\}_{h\in[H]}^{i\in[t-1]}$. We have
\begin{align*}
&\Eb[W_t(h,g)|\Fc_{t,h}]=\left[\left({g}_h-\#_{k,h}\right)(x_h^t,a_h^t)\right]^2, \\
&\Var[W_t(h,g)|\Fc_{t,h}]\leq 36H^2\Eb[W_t(h,g)|\Fc_{t,h}].
\end{align*}
By Freedman's inequality, with probability at least $1-\delta$,
\begin{align*}
    &\left|\sum_{t=1\lor (k-w)}^{k}W_t(h,g)-\sum_{t=1\lor (k-w)}^{k}\left[\left({g}_h(x_h^t,a_h^t)-\#_{k,h}\right)(x_h^t,a_h^t)\right]^2\right| \\
    &\qquad \leq \Oc\left(H\sqrt{\log(1/\delta)\sum_{t=1\lor (k-w)}^{k}\left[\left({g}_h(x_h^t,a_h^t)-\#_{k,h}\right)(x_h^t,a_h^t)\right]^2  }+\log(1/\delta)\right).
\end{align*}
Taking union bound over $[K]\times[H]\times\Gc$,
\begin{align*}
    &\left|\sum_{t=1\lor (k-w)}^{k}W_t(h,g)-\sum_{t=1\lor (k-w)}^{k}\left[\left({g}_h(x_h^t,a_h^t)-\#_{k,h}\right)(x_h^t,a_h^t)\right]^2\right| \\
    &\qquad \leq \Oc\left(H\sqrt{\iota\sum_{t=1\lor (k-w)}^{k}\left[\left({g}_h(x_h^t,a_h^t)-\#_{k,h}\right)(x_h^t,a_h^t)\right]^2  }+\iota\right),
\end{align*}
where $\iota=\log(HK|\Gc|/\delta)$. We have
\begin{align*}
&-\sum_{t=1\lor (k-w)}^{k}W_t(h,g)  \\
&\leq -\sum_{t=1\lor (k-w)}^{k}\left[\left({g}_h(x_h^t,a_h^t)-\#_{k,h}\right)(x_h^t,a_h^t)\right]^2  +\Oc\left(H\sqrt{\iota\sum_{t=1\lor (k-w)}^{k}\left[\left({g}_h(x_h^t,a_h^t)-\#_{k,h}\right)(x_h^t,a_h^t)\right]^2  }+\iota\right) \\
&\leq \Oc(H^2\iota). 
\end{align*}
I.e., 
\begin{align*}
&\sum_{t=1\lor (k-w)}^{k}\left[\#_{k,h}(x_h^t,a_h^t)-r_h^t-\max_{a'\in\mathcal{A}}Q_{h+1;\stark}(x_{h+1}^t,a')\right]^2 \\
&\quad \qquad \leq \sum_{t=1\lor (k-w)}^{k}\left[{g}_h(x_h^t,a_h^t)-r_h^t-\max_{a'\in\mathcal{A}}Q_{h+1;\stark}(x_{h+1}^t,a')\right]^2 + \Oc(H^2\iota).
\end{align*}
Therefore,
\begin{align*}
    &\sum_{t=1\lor (k-w)}^{k}\left[Q_{h;(*,k)}(x_h^t,a_h^t)-r_h^t-\max_{a'\in\mathcal{A}}Q_{h+1;\stark}(x_{h+1}^t,a')\right]^2 \\
    &\leq \sum_{t=1\lor (k-w)}^{k}\left[\#_{k,h}(x_h^t,a_h^t)-r_h^t-\max_{a'\in\mathcal{A}}Q_{h+1;\stark}(x_{h+1}^t,a')\right]^2 +2H^2\Delta_P^w(k,h)+2H\Delta_R^w(k,h) \\
    &\leq \sum_{t=1\lor (k-w)}^{k}\left[{g}_h(x_h^t,a_h^t)-r_h^t-\max_{a'\in\mathcal{A}}Q_{h+1;\stark}(x_{h+1}^t,a')\right]^2 + 2H^2\Delta_P^w(k,h)+2H\Delta_R^w(k,h)+\Oc(H^2\iota),
\end{align*}
where the first inequality follows from Lemma~\ref{lemma:square-dif-2} and the definition of $\Delta_P^w$ and $\Delta_P^w$. By the definition of $\Bc^{k}$ and $\beta=cH^2\log\frac{KH|\Gc|}{\delta}$ with some large absolute constant $c$, we conclude that with probability at least $1-\delta$, $Q_{(*,k)}\in\Bc^{k}$ for all $k\in[K]$.
\end{proof}

\begin{lemma}\label{lemma:concentration-2}
If $\beta=cH^2\log\frac{KH|\Gc|}{\delta}$, then with probability at least $1-\delta$, for all $(k,h)\in[K]\times[H]$, we have
\begin{align*}
&\sum_{t=1\lor (k-w-1)}^{k-1}\left[f_h^{k}(s_h^t,a_h^t)-r_h^{k-1}(s_h^t,a_h^t) -\underset{x'\sim P_h^{k-1}(x_h^t,a_h^t)}{\Eb}\max_{a'\in\mathcal{A}}f_{h+1}^k(s',a')\right]^2 \\
&\qquad \qquad \qquad\qquad \qquad \qquad\qquad \qquad \qquad \leq 6H^2\Delta_P^w(k-1,h)+6H\Delta_R^w(k-1,w)+\Oc(\beta).
\end{align*}
\end{lemma}
\begin{proof}
Define
\begin{align*}
    \#_{k,h}^f(x_h^t,a_h^t)= \Eb[r_h^{t}(s_h^t,a_h^t)]+\underset{x'\sim P_h^t(x_h^t,a_h^t)}{\Eb}\max_{a'\in\mathcal{A}}f_{h+1}(s',a').
\end{align*}
Fix a tuple $(k,h,f)\in[K]\times[H]\times\Gc$. Let
\begin{align*}
W_t(h,f):&=\left[f_h(x_h^t,a_h^t)-r_h^{t}-\max_{a'\in\mathcal{A}}f_{h+1}(x_{h+1}^t,a')\right]^2-\left[\#_{k,h}^f(x_h^t,a_h^t)-r_h^{t}-\max_{a'\in\mathcal{A}}f_{h+1}(x_{h+1}^t,a')\right]^2 \\
&=[f_h(x_h^t,a_h^t)-\#_{k,h}^f(x_h^t,a_h^t)]\left[f_h(x_h^t,a_h^t)+\#_{k,h}^f(x_h^t,a_h^t)-2\left(r_h^t+\max_{a'\in\mathcal{A}}f_{h+1}(x_{h+1}^t,a')\right)\right]
\end{align*}
and $\Fc_{t,h}$ be the filtration induced by $\{x_1^i,a_1^i,\cdots,x_H^i\}_{i\in[t-1]}\cup\{x_1^t,a_1^t,\cdots,x_h^t,a_h^t\}\cup\{r_h^i\}_{h\in[H]}^{i\in[t-1]}$. We have
\begin{align*}
&\Eb[W_t(h,f)|\Fc_{t,h}]=\left[\left({f}_h-\#_{k,h}^f\right)(x_h^t,a_h^t)\right]^2, \\
&\Var[W_t(h,f)|\Fc_{t,h}]\leq 36H^2\Eb[W_t(h,g)|\Fc_{t,h}].
\end{align*}
By Freedman's inequality, we have
\begin{align*}
    &\left|\sum_{t=1\lor (k-w)}^{k}W_t(h,f)-\sum_{t=1\lor (k-w)}^{k}\left[({f}_h-\#_{k,h}^f)(x_h^t,a_h^t)\right]^2\right|  \\
    &\qquad\leq \Oc\left(H\sqrt{\log(1/\delta)\sum_{t=1\lor (k-w)}^{k}\left[({f}_h-\#_{k,h}^f)(x_h^t,a_h^t)\right]^2  }+\log(1/\delta)\right).
\end{align*}
Taking union bound over $[K]\times[H]\times\Gc$, we have
\begin{align*}
    &\left|\sum_{t=1\lor (k-w)}^{k}W_t(h,g)-\sum_{t=1}^{k}\left[({f}_h-\#_{k,h}^f)(x_h^t,a_h^t)\right]^2\right| \leq \Oc\left(H\sqrt{\iota\sum_{t=1\lor (k-w)}^{k}\left[({f}_h-\#_{k,h}^f)(x_h^t,a_h^t)\right]^2  }+\iota\right),
\end{align*}
where $\iota=\log(KH|\Gc|/\delta)$.

Note that
\begin{align*}
    &\sum_{t=1\lor (k-w-1)}^{k-1}W_t(h,f^k) \\
    &=\sum_{t=1\lor (k-w-1)}^{k-1}\left[f_h^k(x_h^t,a_h^t)-r_h^{t-1}-\max_{a'\in\mathcal{A}}f_{h+1}^k(x_{h+1}^t,a')\right]^2 \\
    &\qquad \qquad  -\sum_{t=1\lor (k-w-1)}^{k-1}\left[\#_{k-1,h}^{f^k}(x_h^t,a_h^t)-r_h^{t-1}-\max_{a'\in\mathcal{A}}f_{h+1}^k(x_{h+1}^t,a')\right]^2 \\
    &\leq\sum_{t=1\lor (k-w-1)}^{k-1}\left[f_h^k(x_h^t,a_h^t)-r_h^{t-1}-\max_{a'\in\mathcal{A}}f_{h+1}^k(x_{h+1}^t,a')\right]^2 \\
    &\qquad   -\sum_{t=1\lor (k-w-1)}^{k-1}\left[\Tc_{h}^{k-1}f_{h+1}^k(x_h^t,a_h^t)-r_h^{t-1}-\max_{a'\in\mathcal{A}}f_{h+1}^k(x_{h+1}^t,a')\right]^2  +2H^2\Delta_P^w(k-1,h)+2H\Delta_R^w(k-1,w) \\
    &\leq \sum_{t=1\lor (k-w-1)}^{k-1}\left[f_h^k(x_h^t,a_h^t)-r_h^{t-1}-\max_{a'\in\mathcal{A}}f_{h+1}^k(x_{h+1}^t,a')\right]^2 \\
    &\qquad \qquad  -\inf_{g\in\Gc}\sum_{t=1\lor (k-w-1)}^{k-1}\left[g_h(x_h^t,a_h^t)-r_h^{t-1}-\max_{a'\in\mathcal{A}}f_{h+1}^k(x_{h+1}^t,a')\right]^2 +2H^2\Delta_P^w(k-1,h)+2H\Delta_R^w(k-1,w)  \\
    &\leq \beta+4H^2\Delta_P^w(k-1,h)+4H\Delta_R^w(k-1,w),
\end{align*}
where the first inequality follows from Lemma~\ref{lemma:square-dif-2} and the definition of $\Delta_P^w$ and $\Delta_P^w$, the second inequality follows from Assumption~\ref{assump:completeness-2}, and the last inequality follows from the definition of $\Bc^{k-1}$.

Therefore, 
\begin{align*}
   \sum_{t=1\lor (k-w-1)}^{k-1}\left[(f_h^k-\#_{k-1,h}^{f^k})(x_h^t,a_h^t)\right]^2\leq \beta+4H^2\Delta_P^w(k-1,h)+4H\Delta_R^w(k-1,w)+\Oc\left(H^2\iota\right).
\end{align*}
Finally, we use Lemma~\ref{lemma:square-dif-2} again and obtain
\begin{align*}
    &\sum_{t=1\lor (k-w-1)}^{k-1}\left[(f_{h}^k-\Tc_h^{k-1} f_{h+1}^k)(x_h^t,a_h^t)\right]^2 \\
    &\leq \sum_{t=1\lor (k-w-1)}^{k-1}\left[(f_{h}^k-\#_{k-1,h}^{f^k})(x_h^t,a_h^t)\right]^2+2H^2\Delta_P^w(k-1,h)+2H\Delta_R^w(k-1,w) \\
    &\leq 6H^2\Delta_P^w(k-1,h)+6H\Delta_R^w(k-1,w)+\Oc(\beta).
\end{align*}
\end{proof}

By Lemma~\ref{lemma:Qstar}, with probability at least $1-\delta$, we have
\begin{align*}
   \URM{1}&=\sum_{t=1}^{k}\left(V_{1;(*,t-1)}^{\pi^{(*,t-1)}}-V_{1;(*,t-1)}^{\pi^t}\right)(x_1) \\
    &\leq \sum_{t=1}^{k}\left(\max_{a\in\mathcal{A}}f_{1}^t(x_1,a)-V_{1;(*,t-1)}^{\pi^t}(x_1) \right)\\
    &\leq \sum_{h=1}^{H}\sum_{t=1}^{k}\underset{(x_h,a_h)\sim(\pi^t,(*,t-1))}{\Eb}[(f_h^t-\Tc_h^{t-1}f_{h+1}^t)(x_h,a_h)],
\end{align*}
where the first inequality follows from Lemma~\ref{lemma:Qstar} and the optimistic planning step (line 3) in Algorithm~\ref{Alg:2} which guarantees that $V_{1;(*,k-1)}^*\leq \sup_{a}f_1^k(x_1,a)$ for every episode $k$, the last inequality follows from generalized policy loss decomposition (Lemma~\ref{lemma:dr-decom-2}) and the fact that $\pi^k=\pi_{f^k}$ (line 3 in Algorithm~\ref{Alg:2}).

Now we invoke Lemma~\ref{lemma:GDE} and Lemma~\ref{lemma:concentration-2} with
\begin{align*}
    &\theta=\sqrt{\frac{1}{w}}, C=H, \\
    &\mathcal{X}=\mathcal{S}\times\mathcal{A},\Phi=(I-\Tc_h)\Fc, \mbox{and } \Pi=\Dc_{\Delta,h}, \\
    &\phi_k=f_h^k-\Tc_h^{k-1}f_{h+1}^k, \mu_k=\mathbf{1}\{\cdot=(x_h^k,a_h^k)\}
\end{align*}
and obtain
\begin{align*}
    &\sum_{t=1}^{k}\underset{(x_h,a_h)\sim(\pi^t,(*,t-1))}{\Eb}[(f_h^t-\Tc_h^{t-1}f_{h+1}^t)(x_h,a_h)] \\
    &\leq \sum_{t=1}^{k}(f_h^t-\Tc_h^{t-1} f_{h+1}^t)(x_h^t,a_h^t)+\mathcal{O}\left(\sqrt{k\log(k)}\right) \\
    &\leq \mathcal{O}\left(\frac{k}{w}\sqrt{w\cdot\GBE(\Fc,\mathcal{D}_{\Delta,h},\sqrt{1/w})\left(H^2\log[KH|\mathcal{G}|/\delta]+H^2\sup_{t\in[k]}\Delta_{P}^w(t,h)+H\sup_{t\in[k]}\Delta_{R}^w(t,h)\right)}+\sqrt{w}\right) \\
    &\leq  \mathcal{O}\left(\frac{Hk}{\sqrt{w}}\sqrt{d\log[kH|\mathcal{G}|/\delta]}+\frac{Hk}{\sqrt{w}}\sqrt{d\sup_{t\in[k]}\Delta_{P}^w(t,h)}+\frac{\sqrt{H}k}{\sqrt{w}}\sqrt{d\sup_{t\in[k]}\Delta_{R}^w(t,h)}+\sqrt{w}\right),
\end{align*}
where the second inequality follows from Azuma-Hoeffding inequality, and in the last inequality, we use $\sqrt{a+b}\leq \sqrt{a}+\sqrt{b}$ for any positive $a,b\geq 0$ and we define $d=\GBE(\Fc,\mathcal{D}_{\Delta,h},\sqrt{1/w})$.

Summing over step $h\in[H]$ gives
\begin{align*}
    &\sum_{h=1}^{H}\sum_{k=1}^{k}\underset{(x_h,a_h)\sim(\pi^k,(*,k-1))}{\Eb}[(f_h^k-\Tc_h^{k-1}f_{h+1}^k)(x_h,a_h)] \\
    &\leq \mathcal{O}\left(\frac{H^2k}{\sqrt{w}}\sqrt{d\log[KH|\mathcal{G }|/\delta]}+\frac{H^2k}{\sqrt{w}}\sqrt{d\sup_{t\in[k]}\Delta_{P}^w(t,h)}+\frac{H^{3/2}k}{\sqrt{w}}\sqrt{d\sup_{t\in[k]}\Delta_{R}^w(t,h)}+H\sqrt{w}\right),
\end{align*}
which completes the proof.

\subsection{Proof of \Cref{coro:thm-bandit_fb}}
For ease of exposition, let $d=\GBE(\Fc,\mathcal{D}_{\Delta,h},\sqrt{1/w})$. We adopt average variation $L$ defined in~(\ref{eqn:avg-pathlength}) and average variation $L$ in rewards defined in~(\ref{eqn:avg-pathlength-rewards}). Then we have 
\begin{align*}
    &\sum_{h=1}^H\sum_{t=1}^{K}\underset{(x_h,a_h)\sim(\pi^t,(*,t-1))}{\Eb}[(f_h^t-\Tc_h^{t-1}f_{h+1}^t)(x_h,a_h)] \\
    &\leq \widetilde{\mathcal{O}}\left(\frac{H^2K}{\sqrt{w}}\sqrt{d}\sqrt{\log|\Gc|}+\frac{H^2K}{\sqrt{w}}\sqrt{dLw^2}+\frac{H^{\frac{3}{2}}K}{\sqrt{w}}\sqrt{dL_\theta w^2}+H\sqrt{w}\right) \\
    &\leq \widetilde{\mathcal{O}}\left(H^2K\sqrt{d}\left(\frac{\sqrt{\log|\Gc|}}{\sqrt{w}}+(\sqrt{L}+\frac{\sqrt{L_\theta}}{\sqrt{H}}+\frac{1}{HK\sqrt{d}})\sqrt{w}\right)\right).  
\end{align*}
Note first that $\frac{\sqrt{\log|\Gc|}}{\sqrt{L}+\frac{\sqrt{L_\theta}}{\sqrt{H}}+\frac{1}{HK\sqrt{d}}}>1$ when $|\Gc|>10$. 

If $\frac{\sqrt{\log|\Gc|}}{\sqrt{L}+\frac{\sqrt{L_\theta}}{\sqrt{H}}+\frac{1}{HK\sqrt{d}}}\geq K$, i.e., $\sqrt{L}+\frac{\sqrt{L}_\theta}{\sqrt{H}}\leq \frac{1}{K}\left(\sqrt{\log|\Gc|}-\frac{1}{H\sqrt{d}}\right)$, we select $w=K$ and we have
\begin{align*}
&\sum_{h=1}^H\sum_{t=1}^{K}\underset{(x_h,a_h)\sim(\pi^t,(*,t-1))}{\Eb}[(f_h^t-\Tc_h^{t-1}f_{h+1}^t)(x_h,a_h)] \leq \widetilde{O}\left(H^2K^{\frac{1}{2}}d^{\frac{1}{2}}(\log|\Gc|)^{\frac{1}{2}}\right).
\end{align*}

If $\frac{\sqrt{\log|\Gc|}}{\sqrt{L}+\frac{\sqrt{L_\theta}}{\sqrt{H}}+\frac{1}{HK\sqrt{d}}}< K$, i.e., $\sqrt{L}+\frac{\sqrt{L_\theta}}{\sqrt{H}}> \frac{1}{K}\left(\sqrt{\log|\Gc|}-\frac{1}{H\sqrt{d}}\right)$, we select $w=\lceil\frac{\sqrt{\log|\Gc|}}{\sqrt{L}+\frac{\sqrt{L_\theta}}{\sqrt{H}}+\frac{1}{HK\sqrt{d}}}\rceil$ and we have
\begin{align*}
&\sum_{h=1}^H\sum_{t=1}^{K}\underset{(x_h,a_h)\sim(\pi^t,(*,t-1))}{\Eb}[(f_h^t-\Tc_h^{t-1}f_{h+1}^t)(x_h,a_h)]  \\
&\leq \widetilde{\mathcal{O}}\left(H^2KL^{\frac{1}{4}}d^{\frac{1}{2}}(\log|\Gc|)^{\frac{1}{4}}+H^{\frac{7}{4}}KL_\theta^{\frac{1}{4}}d^{\frac{1}{2}}(\log|\Gc|)^{\frac{1}{4}}+H^{\frac{3}{2}}K^{\frac{1}{2}}d^{\frac{1}{4}}(\log|\Gc|)^{\frac{1}{4}}\right).
\end{align*}

\end{document}